\documentclass[12pt]{article}
\usepackage[table]{xcolor} 
\usepackage{colortbl} 
\usepackage{amsfonts}
\usepackage{amssymb}
\usepackage{amsmath}
\usepackage{amsthm}
\usepackage{graphicx}
\usepackage{dsfont}
\usepackage{subcaption}
\usepackage{array,multirow}
\usepackage{hyperref}
\usepackage{lineno}
\usepackage{enumitem}
\usepackage{geometry}
\usepackage{algorithm}  
\usepackage{algorithmic}  
\usepackage{etoolbox}
\usepackage{mathtools}
\usepackage{mathtools}
\usepackage{bm}
\usepackage{makecell}
\usepackage{tikz}
\usepackage{longtable}
\usepackage{multirow}
\usepackage{colortbl}
\usepackage{xcolor}
\usepackage{array}
\usepackage[square,numbers]{natbib}
\usepackage{longtable}
 \usepackage{tabularx}
\usepackage{array}
\usepackage{multirow}
\usepackage{makecell}
 \usepackage{amsthm}
 
 \newcommand{\comb}[1]{\textcolor{black}{#1}}

 \newcommand{\combl}[1]{\textcolor{black}{#1}}
  \newcommand{\comblu}[1]{\textcolor{black}{#1}}

\theoremstyle{definition}
\newtheorem{example}{Example}
 \newtheorem{theorem}{Theorem}

\begin{document}
\begin{center}
		{\Large ReLU Networks for Exact Generation of Similar Graphs}\bigskip\\
Mamoona Ghafoor
~and
 Tatsuya Akutsu


  {Bioinformatics Center, Institute for Chemical Research, Department of Intelligence Sciences and Technology, 
Kyoto University, Uji 611-0011, Japan }

{Email: mamoona.ghafoor@kuicr.kyoto-u.ac.jp; takutsu@kuicr.kyoto-u.ac.jp
}\\
		\today	
	\end{center}
\begin{abstract}
Generation of graphs constrained by a specified graph edit distance from a source graph is important in applications such as cheminformatics, network anomaly synthesis, and structured data augmentation. Despite the growing demand for such constrained generative models in areas including molecule design and network perturbation analysis, the neural architectures required to provably generate graphs within a bounded graph edit distance remain largely unexplored. In addition, existing graph generative models are predominantly data-driven and depend heavily on the availability and quality of training data, which may result in generated graphs that do not satisfy the desired edit distance constraints.
In this paper, we address these challenges by theoretically characterizing ReLU neural networks capable of generating graphs within a prescribed graph edit distance from a given graph. In particular, we show the existence of constant depth and 
$\mathcal{O}(n^2 d)$ size ReLU networks that deterministically generate graphs within edit distance $d$ from a given input graph with $n$ vertices, thereby eliminating the reliance on training data and guaranteeing the validity of the generated graphs.
Experimental evaluations demonstrate that the proposed network successfully generates valid graphs for instances with up to 1400 vertices and edit distance bounds up to 140, whereas the baseline generative models GraphRNN and GraphGDP fail to generate any graph with the desired graph edit distance. 
These results, supported by experiments demonstrating both scalability and exactness of the proposed networks, provide a theoretical foundation for constructing compact generative models with guaranteed validity, offering a new paradigm for graph generation that moves beyond probabilistic sampling toward exact synthesis under similarity constraints.
An implementation of the proposed networks is available at~\url{https://github.com/MGANN-KU/GraphGen\_ReLUNetworks}.
\end{abstract}

\section{Introduction}
Generative networks, a powerful class of machine learning algorithms, excel at capturing the underlying patterns and distributions of training data. This fundamental ability to model complex data structures allows them to synthesize new, realistic samples, making them invaluable beyond data generation for tasks like imputation and representation learning~\cite{WC2019, VD2016, AN2022}. These models have seen widespread adoption across numerous fields. Their applications range from natural language processing and data augmentation to DNA sequence synthesis and drug discovery~\cite{KN2017, TB2013, BA2017, GE2022}. The impact on bioinformatics is particularly notable, where they are employed for molecule generation, motif discovery, drug discovery, cancer research, secondary structure prediction, and single-cell RNA sequencing analysis~\cite{ZJ2014, LJ2018, SK2020, GC2020}.

Generative models comprise a diverse set of architectures. Autoencoders like VAEs~\cite{KD2019} and DAEs~\cite{BY2013} learn compressed data representations through encoding and decoding. Generative adversarial networks (GANs)~\cite{GI2020}, including deep convolutional GAN~\cite{GF2018}, use a competitive generator-discriminator framework to create realistic data. Probabilistic models like deep Boltzmann machines capture complex data distributions~\cite{HG2009}, while autoregressive models (e.g., PixelCNN~\cite{VO2016}, PixelRNN~\cite{VD2016b}) generate data sequentially. The deep recurrent attentive writer (DRAW) model~\cite{GK2015} integrates recurrent neural networks with attention mechanisms to generate images by focusing on specific regions of the data during the generation process.
Recently, diffusion models have become prominent. They work by learning to reverse a gradual noising process; starting from pure noise, a neural network is trained to iteratively denoise the data to generate coherent samples~\cite{HJ2020, YL2023}.

The selection of a machine learning model's function family and architecture is a critical yet non-trivial trade-off. An overly broad family risks overfitting and high computational cost, while an overly constrained one may lack predictive accuracy, with no universal solution for all problems~\cite{KA2022}. Although the universal approximation theorem guarantees that even shallow networks can approximate complex functions, this often necessitates an infeasibly large number of nodes~\cite{HSW1989}. Research has thus shifted to the efficiency of different architectures, showing that depth exponentially enhances representational power for certain functions~\cite{MPCB2014, RPKGS2017, T2015}. For example, deep networks can model periodic functions more effectively, with specific width-depth trade-offs established~\cite{SM2014, CNPW2019}.

The choice of activation function also determines expressive capacity. Networks with piecewise linear activations, for instance, do not exponentially increase their complexity~\cite{BR2019}. Furthermore, studies have proven that neural networks with various activations (sigmoidal, tanh, ReLU) can approximate other model classes like decision trees and random forests~\cite{KA2022, YDS2010, BSW2019}. Recent work has even demonstrated the existence of compact generative networks with ReLU activations capable of producing strings within a specific edit distance~\cite{MT2024}.
Graph edit distance (GED) introduced by Sanfeliu and Fu~\cite{Sanfeliu2012} in 1983 emerged as an important tool for measuring similarity between graphs, extending the notion of string and tree edit distances to general graph structures. Since then, GED has found widespread applications in pattern recognition, computer vision, bioinformatics, chemoinformatics, and network analysis~\cite{Gao2010Full, Ibragimov2013, Riesen2015}. The exact computation of GED was later shown to be an NP-hard problem \cite{Zeng2009}, which motivated extensive research on approximate solutions. Early approaches relied on exact exponential-time algorithms or heuristic methods. To improve efficiency, several works reformulated GED as a quadratic assignment problem enabling the use of combinatorial optimization techniques and approximate solvers~\cite{Bougleux2015}. More recently, approximate and learning-based methods, including continuous relaxations and neural network approaches, have been proposed to scale GED computation to larger graphs while maintaining reasonable accuracy \cite{Bai2020, Bai2018}.

Recent advances in structured generative networks have produced powerful models that balance empirical performance with structural validity. Autoregressive approaches, such as GraphRNN~\cite{You2018GraphRNN} and its variants~\cite{Wang2025LearningOrder}, sequentially construct graphs, with some learning dynamic generation orders for state-of-the-art molecular design, while others like AutoGraph~\cite{Chen2025AutoGraph} leverage transformers to frame graph generation as a sequence modeling task. 
In parallel, diffusion models have emerged as a dominant paradigm. 
Frameworks like GraphGDP~\cite{Huang2022GraphGDP} and BetaDiff~\cite{Liu2025BetaDiff} generate permutation-invariant graphs and jointly model discrete and continuous graph attributes. 
A key innovation is the enforcement of hard constraints, as seen in models that use specialized noise mechanisms to guarantee properties like planarity and acyclicity throughout the diffusion process~\cite{Madeira2024StructGraph}.

Despite these empirical successes, a fundamental limitation persists. Current models provide probabilistic guarantees that are inherently dependent on the quality and coverage of their training data. Consequently, they are unable to perform exact enumeration of the underlying combinatorial space or offer comprehensive guarantees of complete structural validity and coverage~\cite{Verma2025, Bommakanti2024}. This data-driven paradigm leaves a critical gap for applications requiring rigorous, rather than probabilistic, certainty over the space of generated structures. 
Recently, Ghafoor and Akutsu~\cite{MT2024, MT2026} studied the existence of ReLU generative networks to generate similar strings and trees with the desired string and tree edit distance, respectively. 

In this paper, we propose a novel framework for the exact generation of vertex-labeled graphs within a specified graph edit distance from a given graph. We theoretically establish the existence of constant-depth ReLU networks capable of generating graphs that satisfy the desired edit distance constraints, thereby providing formal guarantees that are typically absent in data-driven generative models. We further demonstrate the practical applicability of the proposed approach through extensive experiments on graphs with up to 1400 vertices, and compare it with the baseline generative models GraphRNN~\cite{You2018GraphRNN} and GraphGDP~\cite{Huang2022GraphGDP}.
An implementation of the proposed networks is available at~\url{https://github.com/MGANN-KU/GraphGen\_ReLUNetworks}.

The paper is organized as follows: 
Preliminaries are discussed in Section~\ref{sec:pre}. 
The existence of ReLU networks to generate any graph with graph edit distance at most $d$ due to substitution (resp., deletion, insertion) operations is discussed in Section~\ref{sec:GS} (resp., Section~\ref{sec:GD}, Section~\ref{sec:GI}).  
The generation of any graph with graph edit distance at most $d$ due 
to simultaneous application of deletion, substitution, and insertion operations by using  a ReLU network is discussed in Section~\ref{sec:GED}. 
Computational experiments are discussed in Section~\ref{sec:exp}.
A conclusion and future directions are given in Section~\ref{sec:concl}. 

\section{Preliminaries}\label{sec:pre}
Let $G$ be a vertex-labeled graph with $n$ vertices and labels from  the symbol set 
$\Sigma = \{1, 2, \ldots, m\}$,  
and let $v_1 , v_2,\ldots,v_n$ be an arbitrary sequence of vertices of $G$.
We represent the graph $G$ by using a column matrix $L(G) = (\ell_i)$, called label matrix such that $\ell_i$ is the label of the vertex $v_i$, and an $n \times n$ adjacency matrix $A(G) = (a_{ik})$ such that $a_{ik} = 1$ (resp., 0) if there is an edge between vertices $v_i$ and $v_k$ (resp., otherwise) as shown in Fig.~\ref{fig:G_r}(a) and (b).
When the underlying graph is fixed, then we simply use the notation $L$ and $A$. 
Three edit operations, deletion, substitution, and insertion, can be performed on a vertex-labeled graph.
There are two types of deletion operations: vertex deletion, which removes an isolated vertex, and edge deletion, which removes an edge, as illustrated in Fig.~\ref{fig:G_r}(c).
The substitution operation changes the label of a vertex, as shown in Fig.~\ref{fig:G_r}(d).
Similarly, there are two types of insertion operations: vertex insertion, which adds an isolated vertex, and edge insertion, which adds a new edge, as illustrated in Fig.~\ref{fig:G_r}(e).
Observe that the
\begin{enumerate}
\item deletion operation on a vertex $v_i$ of $G$ can be viewed as the deletion operation at the $i$-th entry of $L$ and $i$-th row and column of $A$, whereas the deletion of an edge between the vertices $v_i$ and $v_k$ can be viewed as simply replacing $1$ by $0$ at the $(i,k)$-th entry of $A$ as illustrated in Fig.~\ref{fig:G_r}(f),
\item substitution operation on a vertex $v_i$ of $G$ can be viewed as the substitution operation at the $i$-th entry of $L$, as illustrated in Fig.~\ref{fig:G_r}(e),
\item insertion operation of a vertex can be viewed as the insertion of a new entry at the end of $L$ and by adding a new row and column with all entries $0$ at the end of $A$, whereas the insertion of an edge between the vertices $v_i$ and $v_k$ is simply replacing $0$ by $1$ at the $(i,k)$-th entry of $A$ as illustrated in Fig.~\ref{fig:G_r}(g).
\end{enumerate}
The graph edit distance between  two vertex-labeled graphs $G$ and $H$ is defined to be $\mathrm{GED}(G, H) = \min_{\tau \in \mathcal{T}(G, H)} \sum_{e \in \tau} c(e)$ where $\mathcal{T}(G, H)$ denotes the set of edit paths transforming $G$ into $H$, and $c(e)$ denotes the cost of each edit operation $e$~\cite{RB2009}.
To ensure that each path of edit operations results in a valid graph without dangling edges, a vertex may be deleted only after all its incident edges have been removed and an edge may be inserted only if its endpoint vertices already exist or have been inserted beforehand~\cite{BS2017}.
Throughout this paper, we use a random sequence $x_1, x_2, \ldots, x_k$, 
$k \geq 2$, with integers $x_j \in [1, n]$, unless stated otherwise, to indicate the indices of the vertices $v_1 , v_2, \ldots, v_n$ in a graph.

For any real numbers $p, q$ and $\theta $, the ReLU function is defined as ${\rm ReLU}(p)=\max(0,p)$, the function $\delta$ is defined as $\delta(p,q)=1$ if $p=q$, and $\delta(p,q)=0$ otherwise; the threshold function $[~]$ is defined as $[p \geqslant \theta]=1$ for $p \geqslant \theta$, and $[p \geqslant \theta]=0$ otherwise; and the Heaviside function $H$ is defined as $H(p)=1$ when $p \geqslant 0$, and $H(p)=0$ otherwise. \combl{For the logical AND operation $a \land b$ between two binary variables $a, b\in \{0, 1\}$, it holds that $a \land b ={\rm ReLU}(a+b-1)$. }
\begin{figure}[H]
	\centering
	\includegraphics[scale = 0.6]{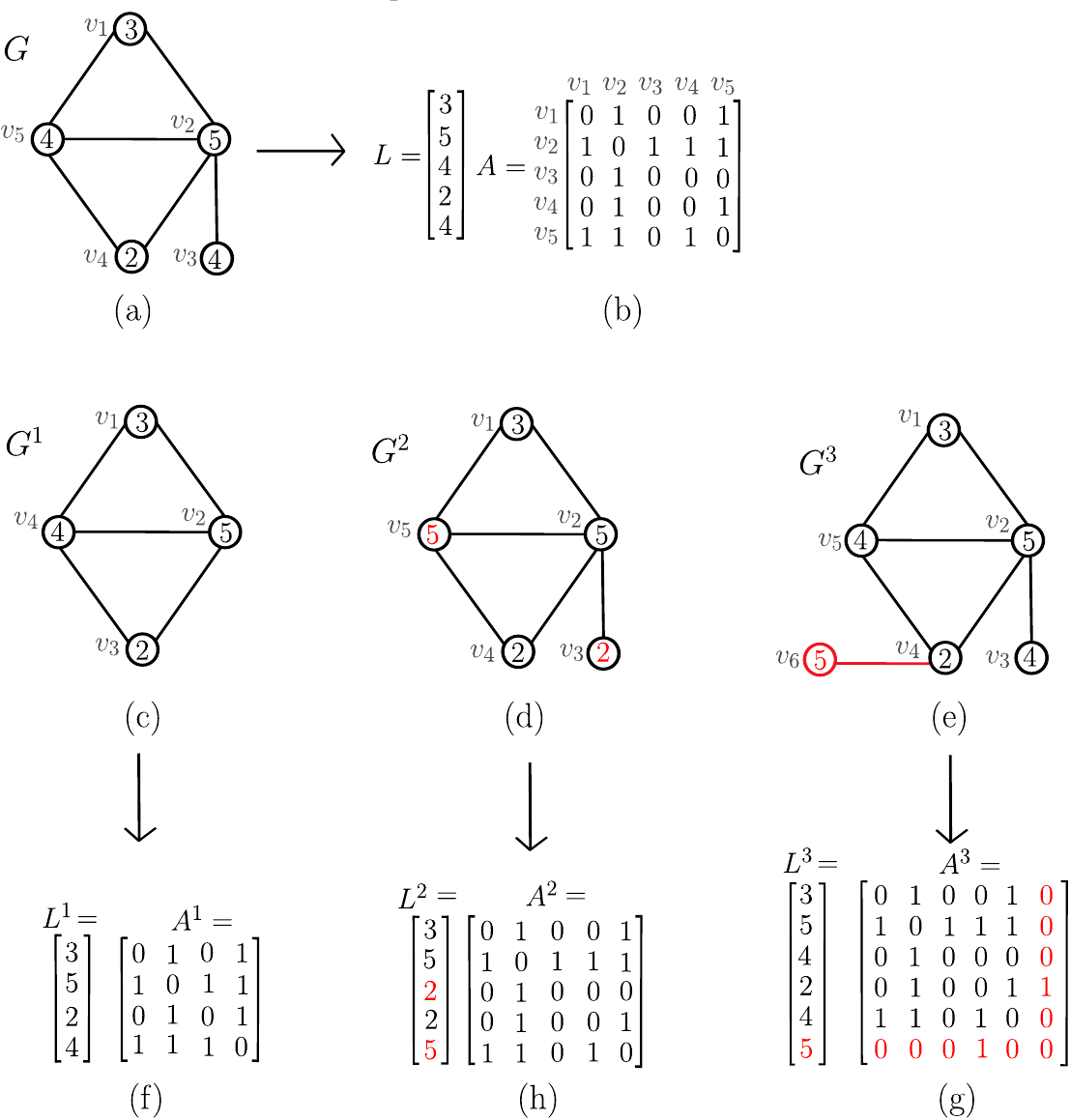}
	\caption{(a) A vertex-labeled graph $G$ with five vertices, an arbitrary sequence  $v_1 , v_2, v_3, v_4,v_5$ of vertices and labels from $\Sigma = \{1, 2,3, 4, 5\}$. 
	(b) The representation of $G$ with the label matrix $L$ and the adjacency matrix $A$. 
	(c) A graph $G^1$ obtained from $G$ after deleting the edge incident to $v_2$ and $v_3$ and then deleting the vertex $v_3$. 
	(d) A graph $G^2$ obtained from $G$ after substituting the label of the vertices  $v_3$ and $v_5$. 
	(e) A graph $G^3$ obtained from $G$ after inserting a vertex $v_6$ with label $5$ and then an edge connecting the vertices  $v_4$ and $v_6$. 
	(f), (g) and (h) are the matrix representations of the graphs obtained in (c), (d) and (e), respectively. Observe that $\text{GED}(G,G^1)=\text{GED}(G,G^2)=\text{GED}(G,G^3)=2$. }\label{fig:G_r}
\end{figure}

\section{GS$_d$-generative ReLU}\label{sec:GS} 
For a vertex-labeled graph $G$ with $n$ vertices, a label set $\Sigma = \{1, 2, \ldots, m\}$ with an arbitrary vertex sequence $v_1, v_2, \ldots, v_n$, 
and a non-negative integer $d$,
we define a {\em GS$_d$-generative ReLU} to be a 
ReLU neural network that generates any graph over $\Sigma$ whose graph edit distance is at most $d$ from $G$ due to the substitution of labels of the vertices indicated by a random sequence $x = x_1, x_2,  \ldots, x_{2d}$, where $ x_j \in \{1, \ldots, n\}$ indicates the index of the vertices and $x_{j+d} \in  \Sigma$
indicates the new label of the vertex $v_{x_j}$. 
Note that the adjacency matrices of $G$ and $G'$ will be the same. 
The existence of the such network is discussed in Theorem~\ref{thm:Gsnn}. 
\begin{theorem}
\label{thm:Gsnn}
For a vertex-labeled graph $G$ with $n$ vertices over 
$\Sigma = \{1, 2, \ldots, m\}$, and a non-negative integer $d$, 
there exists a GS$_d$-generative ${\rm ReLU}$ network with size 
$\mathcal{O}(n^2d)$ and constant depth.
\end{theorem}
\begin{proof}
Suppose $G$ and $G'$ are two vertex-labeled graphs with $n$ vertices over $\Sigma$, with $L=(\ell_i)$ and $L' = (\ell'_i)$ label columns, resp., such that 
$G'$ can be obtained from $G$ with appropriate substitution operations defined by 
a sequence $x = x_1, x_2,  \ldots, x_{2d}$. 
We demonstrate that the process of obtaining $G'$ from $G$ with substitution operations can be simulated with the ReLU function using the following system of equations, where $C \gg  \max(m,n)$, ${j} \in \{1, 2, \ldots, d\}$ and $i \in \{1, 2, \ldots, n\}$.
\begin{align}
e_{j} &= \max \Big( x_j -  C \cdot \sum_{k=1}^{j-1}( \delta(x_j, x_k), 0 ) \Big),   \label{eqe1}\\
f_{i} &= \max \Big( {\ell}_i -  C \cdot \sum_{j=1}^{d}( \delta(e_j, i), 0)  \Big),   \label{eqf1}\\
g_{i} &= \sum_{j=1}^{d} \Big( \max \big(x_{j+d} - C (1- \delta(e_{j}, i) ), 0\big) \Big),   \label{eqg1}\\
{\ell}'_{i} &= f_{i} + g_{i}.   \label{eql1}
\end{align}
The variable $e_j$ ignores repetition in the sequence 
$x_1, x_2, \ldots, x_d$. 
$f_i$ stores the values $\ell_i$ that remain unchanged after substitution operations. 
The variable $g_i$ stores the values that should be substituted.
Finally, the required label column $L' = ({\ell}'_i)$ is obtained by adding $f_i$ and $g_i$, since exactly one of them is non-zero.

The maximum and $\delta$ functions can be simulated using the ReLU activation function, as shown in~\cite[Propositions~1 and~2]{MT2024}. 
As a result, there exists a $GS_d$-generative ${\rm ReLU}$ network of size 
$\mathcal{O}(n^2d)$ and constant depth.
\end{proof}
\begin{example}
\label{exa:sub}
Consider the graph $G$ given in Fig.~\ref{fig:G_r}, $d=3$ and $x=5, 3, 3, 5, 2, 3$, 
\combl{where the first three ($d$) entries correspond to the indices, and the next three entries correspond to the new labels. In Fig.~\ref{fig:G_r}, $x'$ is obtained by removing repeated indices from $x$; for example, index $3$ appears in positions $2$ and $3$, and therefore $x'_3 = 0$ in $x'$ as depicted in red. 
Matrix $F$ stores the labels of the vertices which will remain unchanged in the resultant graph.
$F$ is obtained from $L$ by setting to zero the labels corresponding to the indices that appear in $x'$; for example, $f_3 = 0$ since $3$ appears in $x'_2$. 
All such zeros in $F$ are depicted in red in Fig.~\ref{fig:G_r}.  
The matrix $G$ is obtained by keeping the new labels corresponding to the indices that appear in $x'$ as shown in red. 
Finally, $L'$ is the resulting label matrix after substitution, obtained by adding $F$ and $G$, where the new labels $2$ and $5$ are assigned to the vertices with indices $3$ and $5$, respectively.}
The resultant graph $G'$  obtained by applying the substitution operations on $G$ due to the given $x$ is shown in Fig.~\ref{fig:Sub_N}.
We demonstrate the process of obtaining $L'$ for $G'$ by using Eqs.~(\ref{eqe1})-~(\ref{eql1}) as follows. 
\begin{figure}[H]
	\centering
	\includegraphics[scale = 0.55]{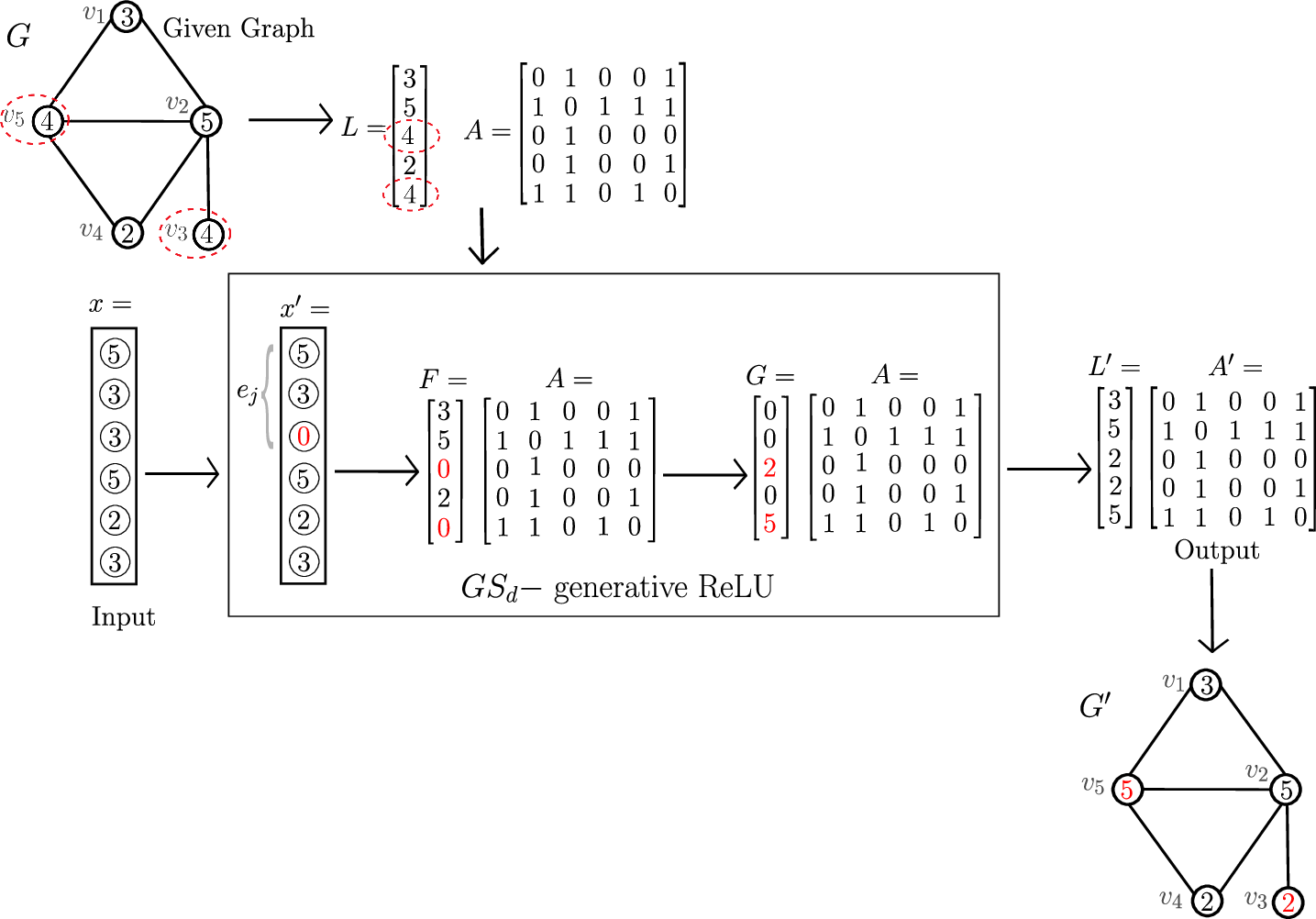}
	\caption{An illustration of simulation of GS$_d$-generative ReLU network for a graph $G$, $d=3$ and the input layer 
$x=5, 3, 3, 5, 2, 3$.  
The graph $G$ is represented by a column vector of labels $L$ and an adjacency matrix $A$, while the output graph $G'$ is generated as a column vector of labels $L'$ and an adjacency matrix $A'$. 
}\label{fig:Sub_N}
\end{figure}
\begin{longtable}{c p{14cm}}
$x_j$ & Indicates vertex index when $j \leq d$, and labels to be substituted otherwise (when $j \geq d+1$)\\
$e_j$ & Eliminates repeated entries from $x$ for $j \in \{1, \dots, d\}$ by setting duplicate values to $0$. 
In this example, $e_3=0$, while $e_1=x_1$ and $e_2=x_2$.\\
$f_{i}$  & Stores the labels that will not be changed, i.e., \combl{ $f_i = 0$ if index $i$ appears in $e$, otherwise $f_i = \ell_i$}. In this case, $F = (f_i)=[3, 5, 0, 2, 0]$, as shown in Fig.~\ref{fig:Sub_N}, \combl{where 
$f_1 = 3$ (resp., $f_3 = 0$) since index $1$ (resp., 3) does not appear (resp., appears) in $x$. }\\
$g_{i}$  & Stores the labels that will be substituted, i.e., \combl{ $g_i = 0$ if index $i$ does not appear in $e$, otherwise $g_i = x_{j+d}$ when $e_j = i$}. Therefore, $G = (g_{i})=[0, 0, 2, 0, 5]$ as shown in Fig.~\ref{fig:Sub_N},
\combl{where 
$g_1 = 0$ (resp., $g_3 = 2$) since index $1$ (resp., 3) does not appear in $e$ (resp., appears at $e_2$ and $x_{2+3} = 2$). }\\
${\ell}'_{i}$  &The entries of the resultant label column of the output graph $G'$ by adding $F$ and $G$. In this case, $L' = ({\ell}'_{i})=[3, 5, 2, 2, 5]$.\\
\end{longtable}
\end{example}
\section{GD$_d$-generative ReLU}\label{sec:GD}
For a vertex-labeled graph $G$ with $n$ vertices, a label set $\Sigma = \{1, 2, \ldots, m\}$ with an arbitrary vertex sequence $v_1, v_2, \ldots, v_n$, 
and a non-negative integer $d$, 
we define the {\em GD$_d$-generative ReLU} to be a 
ReLU neural network that generates any graph over $\Sigma$ whose graph edit distance is at most $d$ from $G$ by deleting its edges and/or vertices. 
These deletions are indicated by a random sequence $x = x_1, x_2,  \ldots, x_{2d}$ over $ \{1, \ldots, n\}$, 
where $x_j$ corresponds to the vertex index 
such that the ReLU network deletes
\begin{enumerate}
\item[-] the edge, if it exists, between $v_{x_j }$ and $v_{x_{j+d}}$ when $x_j \neq x_{j+d}$, and
\item[-] the vertex $v_{x_j }$ if $x_j = x_{j+d}$ and  $v_{x_j }$ is an isolated vertex. 
\end{enumerate}
To ensure a fixed number of vertices in the resultant graph, the network pads the label matrix $L$ (resp., adjacency matrix $A$) with $d$ entries (resp., $d$ rows and $d$ columns) of $B$, where $B \gg m$. 
We denote the padded matrices by $U=(u_{i})$ and $V=(v_{ik})$. 
The label matrix $L'=(\ell'_i)$ and adjacency matrix $A'=(a'_{ik})$ of the resultant graph 
$G'$ can be obtained by removing $B$s from the matrices generated by the network.
The computation process of such a network is given in  Fig.~\ref{fig:Del_N} by considering the random sequence 
$5, 3, 3, 5, 2, 3$, and $d = 3$. 
The matrix $T'$ in Fig.~\ref{fig:Del_N} shows that the network deletes the edge incident to the vertices $v_2$  and $v_3$  since  $x_2=3$ and $x_5=2$. The matrices $W$ and $V'$ in Fig.~\ref{fig:Del_N} show that the network deletes an isolated vertex $v_3$ since $x_3=x_6=3$,  and the network ignores $x_1=x_4=5$ because $v_{x_1}=v_5$ is not an isolated vertex.
The resultant graph is obtained by removing $B$s accordingly.
The existence and complexity of such a network is discussed in Theorem~\ref{thm:Gdnn}. 
\begin{theorem}
\label{thm:Gdnn}
For a vertex-labeled graph $G$ with $n$ vertices, label set  
$\Sigma = \{1, 2, \ldots, m\}$, and a non-negative integer $d$, 
there exists a GD$_d$-generative ${\rm ReLU}$ network with size 
$\mathcal{O}(n^2d)$ and constant depth.
\end{theorem}
\begin{proof}
Suppose $G$ and $G'$ are two vertex-labeled graphs with $n$ vertices  such that $G'$ can be obtained from $G$ by deleting the edges and vertices indicated by a random sequence $x_1, x_2, \ldots, x_{2d}$. 
We claim that the process of deleting edges and vertices to obtain $G'$ can be simulated by the ReLU function through the following three systems of equations: the first models edge deletion, the second performs row deletion, and the third deletes columns from the adjacency matrix to execute vertex deletion.
In these systems, $B$ and $C$ are large numbers such that $C \gg B\gg \max(m,n)$, 
 ${j} \in \{1, 2, \ldots, d\}$ and $i, k \in \{1, 2, \ldots, n+d\}$ unless stated otherwise.\bigskip\\
Edge deletion: 
An edge between the vertices with indices $x_j$ and $x_{j+d}$ can be  deleted by using the following two equations. 
\begin{align}
t_{ik}&= \sum_{j=1}^{d} \Big({\rm ReLU} \big( \delta(x_j, i) \land \delta(x_{j+d}, {k})- \delta(x_j, x_{j+d}) \big)+ \nonumber\\
&~~~~~~~  {\rm ReLU} \big(\delta(x_j, {k}) \land \delta(x_{j+d}, i) - \delta(x_j, x_{j+d}) \big)  \Big),   \label{eqb0}\\
t'_{ik}&= {\rm ReLU} (v_{i{k}}- t_{i{k}}).   \label{eqb'0}
\end{align}
$t_{ik}$ identifies the $(i,k)$-th entry of the padded adjacency matrix $V$ of the graph $G$, which is intended to be set to $0$. \bigskip\\
$t'_{ik}$ stores the entries after edge deletion. 

{Deletion of  rows:}
To delete the vertices with indices $x_j$ such that $x_j=x_{j+d}$ for $j \in \{1, 2, \ldots, d\}$, we must check whether the vertex corresponding to $x_j$ is an isolated vertex in $T'=(t'_{ik})$. 
A vertex is isolated if the sum of all entries in its corresponding row/column is $0$. 
To remove an isolated vertex, first delete the rows corresponding to these vertices by using the following system of equations.
\begin{align}
t''_{i} &=\sum_{k=1}^{n}  t'_{i{k}} \text{~for~} i \in \{1, 2, \ldots, n\},   \label{eqt''0}\\
x'_{j} &= \max \Big( x_j - C \big(1-\sum_{i=1}^{n} (  \delta(x_j, i)\land  \delta(t''_i, 0) \big)  , 0\Big),   \label{eqx'0}\\
e_{ik} &=\max  \Big( 1 -  \sum_{j=1}^{d}\big( \delta(x'_j, x_{j+d}) \land \delta(x'_j, i) \big), 0 \Big),   \label{eqe0}\\
e'_i &=\max  \Big( 1 - \sum_{j=1}^{d} \big( \delta(x'_j, x_{j+d}) \land \delta(x'_j, i) \big), 0 \Big),  \label{eqe'0}\\
f_{ik} &= \max \Big( B \sum_{j=1}^{i} e_{ij} - C \cdot \delta(e_{ik}, 0), 0 \Big),   \label{eqf0}\\
f'_i &= \max \Big( B \sum_{j=1}^{i} e'_{i} - C \cdot \delta(e'_i, 0), 0 \Big),   \label{eqf'0}\\
g^{j}_{ik} &=  \Big[ iB \leq f_{(i+j-1)k} \leq iB + 1 \Big] \text{~for~} j \in \{1, 2, \ldots, d+1\}, \nonumber\\
&~~~~~~~ i\in \{1, 2, \ldots, n\},  \label{eqg0}\\
g'^{j}_i &=  \Big[ iB \leq f'_{i+j-1} \leq iB + 1 \Big] \text{~for~} j \in \{1, 2, \ldots, d+1\}, \nonumber\\
&~~~~~~~ i\in \{1, 2, \ldots, n\},   \label{eqg'0}\\
h^{j}_{ik} &= \max \Big( t'_{(i+j-1) k} - C(1 - g^{j}_{ik}), 0 \Big) \text{~for~} j \in \{1, 2, \ldots, d+1\}, \nonumber\\
&~~~~~~~ i\in \{1, 2, \ldots, n\},   \label{eqh0}\\
h'^{j}_{i} &= \max \Big( u_{(i+j-1)} - C(1 - g'^{j}_i), 0 \Big) \text{~for~} j \in \{1, 2, \ldots, d+1\}, \nonumber\\
&~~~~~~~ i\in \{1, 2, \ldots, n\},   \label{eqh'0}
\end{align}
\begin{align}
w_{ik} &= \sum_{j=1}^{d} h^{j}_{ik}  \text{~for~}  i\in \{1, 2, \ldots, n\},   \label{eqw0}\\
u'_{i} &= \sum_{j=1}^{d} h'^{j}_i  \text{~for~} i \in \{1, \ldots, n\}.  \label{eqw'0}
\end{align}
Vertex $v_i$ is isolated if and only if $t''_{i}=0$.
If $x_j = i$ does not correspond to an isolated vertex, it represents an invalid input and is ignored by assigning $x'_j=0$. 
The variables $e_{ik}$ and ${e}_i'$ indicate the rows that need to be preserved in order to obtain the resultant adjacency and label  matrices $V'$ and $U'$, respectively. 
$f_{ik}$ and ${f}_i'$ assign weights to the preserved rows. 
$g^{j}_{ik}$, $h^{j}_{ik}$ and ${g'}^{j}_i$, ${h'}^{j}_i$ are used to determine the positions of rows to be preserved in $V$ and $U$ respectively.
 $w_{ik}$ and $u'_{i}$ give the matrices after deleting the specified rows from $V$ and $U$, respectively.\bigskip\\
{Deleting columns: }
In order to complete the deletion operation of the vertices with index $x_j$ satisfying $x_j=x_{j+d}$, next eliminate the corresponding columns by  applying the following equations.
\begin{align}
p_{ik} &= \max \Big( 1 - \sum_{j=1}^{d}\big( \delta(x'_j, x_{j+d}) \land \delta(x'_j, k) \big) , 0 \Big) \text{~for~}  i \in \{1, \dots, n\}, \label{eqp0} \\
q_{ik} &= \max \Big( B \sum_{j=1}^{k} p_{ij} - C \cdot \delta(p_{ik}, 0), 0\Big)
\text{~for~} i \in \{1, \dots, n\}, \label{eqq0} \\
r^{j}_{ik} &=  \Big[ {k}B \leq q_{i(k +j-1)} \leq {k}B + 1 \Big] 
\text{~for~} j \in \{1, \dots, d+1\},  i, {k} \in \{1, \dots, n\}, \label{eqr0}  \\
s^{j}_{ik} &= \max\Big( w_{i(k +j-1)} - C(1 - r^{j}_{ik}), 0 \Big)
\text{~for~} i, k \in \{1, \dots, n\},  j \in \{1, \dots, d+1\}, \label{eqs0} \\
 v'_{ik} &= \sum_{j=1}^{d'} s^{j}_{ik} \text{~for~}  i, {k}\in \{1, 2, \ldots, n\}. \label{eqy0}
\end{align}
$p_{ik}$ and $q_{ik}$ specify and assign weights to the columns to be preserved in the padded adjacency matrix $V'$ of the resulting graph, respectively.
$r^{j}_{ik}$ and $s^{j}_{ik}$ indicate the positions of columns to be maintained to get $V'$. 
Finally, $V' = (v'_{ik})$ is the matrix obtained by deleting columns from $W = (w_{ik})$. 
By removing $B$ from $U' = (u'_{i})$  and $V' = (v'_{ik})$, we get the label and adjacency matrices $L'$ and $A'$, resp., of the required graph $G'$.

The maximum function, the $\delta$ and the threshold functions used in the preceding equations can be effectively simulated by the ReLU activation function, by using~\cite[Propositions~1 and~2]{MT2024}. 
As a result, a GD$_d$-generative ${\rm ReLU}$ network of size 
$\mathcal{O}(n^2d)$ and constant depth can be constructed.
\end{proof}
The simulation process of GD$_d$-generative ${\rm ReLU}$ network is demonstrated in Example~\ref{exa:Del}. 
\begin{example}
\label{exa:Del}
Consider the graph $G$ given in Fig.~\ref{fig:G_r}, $d=3$ and a random sequence $x=5, 3, 3, 5, 2, 3$
\combl{which indicates the deletion of vertices or edges depending on the condition $x_j = x_{j+d}$; for example, $x_2 = 3 \neq 2 = x_5$ indicates the deletion of the edge between $v_3$ and $v_2$. The edges to be deleted are depicted in red in $G$ and $A$.
$U$ is obtained by padding three $B$s to $L$ to make its size $n+d = 5+3$, so that the size of the network output remains the same.
$T$ is obtained by setting the encircled entries of $A$ to zero (i.e., performing edge deletion) and padding three rows and three columns with entries $B$ to make its dimension $8 \times 8$. In other words, $T'$ is obtained by performing all edge deletion operations and the padding operation.
Vertex deletion is then performed in three steps:
(i)~check the condition $x_j = x_{j+d}$. In this case, $x_3 = x_6 = 3$, which means that vertex 3 is a candidate for deletion;
(ii)~check in $T'$ whether the row (3rd) corresponding to the vertex 3 has all entries either zero or $B$, which is true in this case. This confirms that vertex 3 is an isolated vertex and should be deleted; hence, its row is marked in red in $T'$ with the corresponding entry in $U$. Note that the same condition does not hold for vertex 5;
(iii)~delete the rows from $T'$ and the labels from $U$ corresponding to the vertices to be deleted to obtain $U'$ and $W$, and finally obtain $V'$ by deleting the columns corresponding to those vertices from $W$. In this case, vertex 3 has been deleted.
To keep the size of the network output the same, additionally delete padded rows and columns from $T'$ and $W$ so that the total deletion count equals $d$. In this case, the number of vertex deletions is 1; therefore, two padded rows and columns marked in red in $T'$ and $W$ are removed to obtain $U'$ and $V'$.
The matrices $L'$ and $A'$ of the resultant graph $G'$ are obtained by performing the unpadding operation.} 
The resultant graph $G'$ obtained by applying the deletion operations on $G$ due to $x$ is shown in Fig.~\ref{fig:Del_N}.
We demonstrate the process of obtaining $L'$ and $A'$ of $G'$ by using Eqs.~(\ref{eqb0})-~(\ref{eqy0}) as follows. 
\begin{figure}[H]
	\centering
	\includegraphics[scale = 0.55]{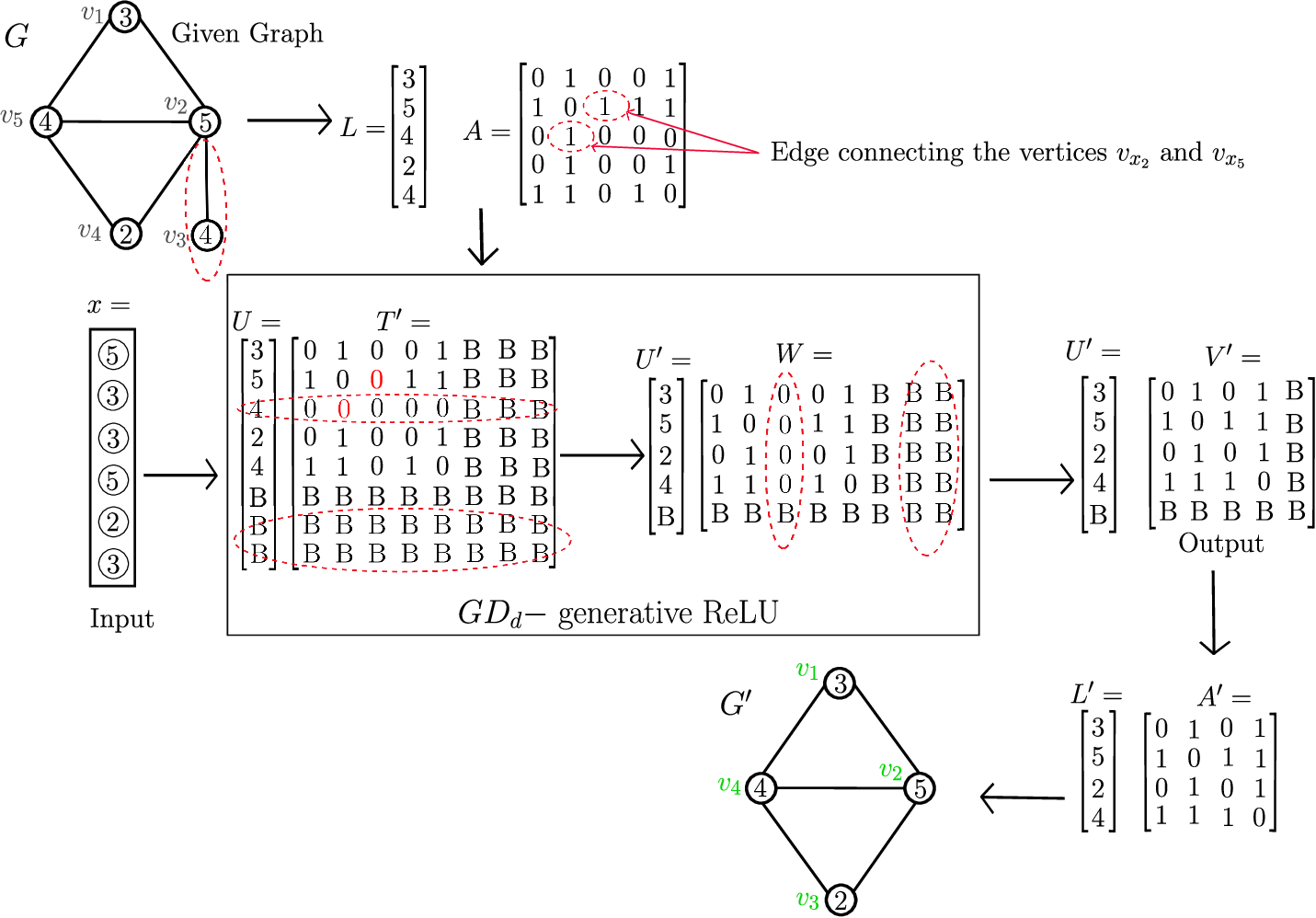}
	\caption{An illustration of the simulation of  GD$_d$-generative ReLU for a given graph $G$ with an arbitrary vertex sequence 
	$v_1, \ldots, v_5$, $d=3$ and the input layer 
$x=5, 3, 3, 5, 2, 3$.  
$U'$ and $W$ are obtained from $U$ and $T'$, resp., by row deletion.
Graph $G$ is represented by the label matrix $L$ and the adjacency matrix $A$, whereas the output graph $G'$ is represented by the label matrix $L'$ and the adjacency matrix $A'$. The entries in red in $T'$ indicate the deleted edges, whereas 
the red ellipses indicate the rows and columns to be deleted to obtain the desired graph $G'$.
}\label{fig:Del_N}
\end{figure}
\begin{longtable}{c p{14cm}}
$x_j$ & Identifies the vertices and edges to be deleted. In this case $x=5, 3, 3, 5, 2, 3$. The values $x_2 = 3$ and $x_5 = 2$ refer to the deletion of the edge
between $v_3$ and $v_2$, while $x_1 = x_4=5$ and $x_3 = x_6=3$ refer to the deletion of vertices $v_5$ and $v_3$. The row and column corresponding to $x_3 = x_6=3$ are deleted because $v_3$ becomes an isolated vertex after edge deletion as shown in the matrix $T'$  in Fig.~\ref{fig:Del_N}. However, the case $x_1 = x_4=5$ is ignored because $v_5$ is not an isolated vertex. \\
$t_{i{k}}$  & A binary variable that equals $1$ when the corresponding entry $v_{ik}$ in the padded adjacency matrix $V$ is set to $0$, indicating that the edge between vertices $v_i$ and $v_k$ should be removed. 
In this case $t_{3,2}=t_{2,3}=1$. All other values are $0$. \\
$t'_{i{k}}$  & The entries of the modified $V$ after setting indicated entries to zero. 
In this example, $t'_{3,2}=t'_{2,3}=0$. All other values of $t'_{i{k}}$ are same as $v_{i{k}}$. The matrix $T'=(t'_{i{k}})$ is shown in Fig.~\ref{fig:Del_N}, where $t'_{3,2}=t'_{2,3}=0$ are in red color.\\
$t''_{i}$  & After edge deletion, this variable checks if $v_i$ is an isolated vertex using $T'=(t'_{i{k}})$. 
More precisely, the $i$-th vertex is isolated if and only if $t''_{i}=0$. 
In this example, $t''_{1}=t''_{4}=2$, $t''_{2}=t''_{5}=3$, $t''_{3}=0$, which shows that only the third vertex is isolated in the matrix $T'$.\\
$x'_{j}$  & Sets the input $x_j=0$ if it corresponds to a non-isolated vertex. Specifically, when $t''_{x_j} > 0$, the vertex cannot be deleted, so $x'_{j}=0$. Hence, for this case $x'_{1}=0$.\\
$e_{ik}$  & A binary variable that takes the value $e_{ik}=1$ when the $i$-th row is retained to keep the vertex $v_i$, otherwise $e_{ik} = 0$ for all $k \in \{1, \dots, n+2\}$. 
In this example $e_{3k}=0$ for $k \in \{1, \dots, n+2\}$. All other entries $e_{ik}=1$ for $i\neq3$.
\\
$e'_i$ & A binary variable that takes the value $e'_i=1$ when the $i$-th label is retained to keep the vertex $v_i$, otherwise $e'_{i} = 0$. For instance, $e'_3=0$ because only the label corresponding to the vertex $v_3$ in matrix $U$ needs to be deleted. All other values are $1$.\\
$f_{ik}$ & Assigns weights to the preserved rows, i.e,. the rows for which $e_{ik}=1$ for $k \in \{1, \dots, n+2\}$, in the ascending order.  
For the third row in this example $f_{3k}=0$ because $e_{3k}=0$ whereas $f_{1k}=B$, $f_{2k}=2B$, $f_{4k}=3B$, $f_{5k}=4B$, $f_{6k}=5B$, $f_{7k}=6B$ for $k \in \{1, \dots, n+2\}$.\\
$f'_i $ & Allocates weights to the retained values of $e'_i$ in ascending order. Here, $f'=[B, 2B, 0, 3B, 4B, 5B, 6B]$. 
\\
$g^{j}_{ik}$  & This variable tracks how far each non-deleted row is shifted forward in the resulting matrix. Since at most $d$ non-padded rows can be deleted from the matrix $V=(v_{ik})$, the maximum shift for any row is $d$.
In this case $g^{1}_{1,k}=g^{1}_{2,k}=1$ because there is no row deleted before rows $i=1, 2$, resulting in a shift of $j-1=0$.   
Similarly, $g^{2}_{i,k}=1$ for $i\neq1$, as these rows experience a forward shift of $j-1=1$ due to the deletion of one row before the rows  $i=3, 4, \dots, n+d$. All other values are set to $0$.
 \\
${g'^j}_{i}$  & Similar to ${g^j}_{ik}$, this variable represents the shift in each non-deleted entry of the padded column matrix $U=(u_{i})$. Thus, $g'^{1}_{1}=g'^{1}_{2}=1$ because no entry is deleted before $i=1, 2$, resulting in a shift of $j-1=0$. Likewise, $g'^{2}_{3}=g'^{2}_{4}=g'^{2}_{5}=g'^{2}_{6}=g'^{2}_{7}=1$, since one entry is deleted before $i=3, 4, \dots, n+d$, giving a forward shift of $j-1=1$.  All other values are $0$.\\
$h^{j}_{ik}$ & This variable shows that a given position $(i,k)$ in the  matrix can take the value $t'_{(i+j-1)k}$, for some $j$, depending on the number of zero rows preceding $f_{ik}$ or on the shift on each row given by $g^{j}_{ik}$. In this example $h^{1}_{1k}=t'_{(1+1-1)k}=[0, 1, 0, 0, 1, B, B]$, $h^{1}_{2k}=t'_{(2+1-1)k}=[1, 0, 0, 1, 1, B, B]$, $h^{2}_{3k}=t'_{(3+2-1)k}=[0, 1, 0, 0, 1, B, B]$, $h^{2}_{4k}=t'_{(4+2-1)k}=[1, 1, 0, 1, 0, B, B]$. Whereas,  $h^{2}_{1k}=h^{3}_{1k}=h^{4}_{1k}=h^{2}_{2k}=h^{3}_{2k}=h^{4}_{2k}=h^{1}_{3k}=h^{3}_{3k}=h^{4}_{3k}=h^{1}_{4k}=h^{3}_{4k}=h^{4}_{4k}=h^{1}_{5k}=h^{3}_{5k}=h^{4}_{5k}=[0, 0, 0, 0, 0, 0, 0]$. \\
$h'^{j}_{i}$ & This variable indicates that the value at position $i$ in the resulting vector of labels may be assigned from $u_{(i+j-1)}$, for some $j$, based on how many zero entries precede $f'_{i}$ or on the shift on each entry given by $g'^{j}_{i}$. Here,  $h'^{1}_{1}=3$, $h'^{1}_{2}=5$, $h'^{2}_{3}=2$, $h'^{2}_{4}=4$, and $h'^{2}_{5}=B$. All other values are $0$.\\
$w_{ik}$ & For a fixed $i$ and $k$, by using the non-zero entries of $h^{j}_{ik}$, this variable stores the entries of the resultant matrix $W$ of order $n \times n+d$ obtained by the deletion of appropriate rows.  In this case, $w_{1k}=[0, 1, 0, 0, 1, B, B]$, $w_{2k}=[1, 0, 0, 1, 1, B, B]$, $w_{3k}=[0, 1, 0, 0, 1, B, B]$, $w_{4k}=[1, 1, 0, 1, 0, B, B]$. The matrix $W$ is shown in Fig.~\ref{fig:Del_N}.\\
$u'_{i}$  & For a fixed value of $i$, the non-zero entries of $h'^{j}_{i}$, form $u'_{i}$. This variable constructs a label matrix $U'$ by deleting the required rows as depicted in Fig.~\ref{fig:Del_N}.\\
$p_{ik}$  & For a fixed $k$, it is a binary variable that equals $1$ when the $k$-th column is kept to preserve the vertex $v_i$, otherwise, $p_{ik}=0$. 
In this example $p_{i3}=0$ for $i \in \{1, \dots, n\}$. All other entries $p_{ik}=1$ for $k\neq3$.\\
$q_{ik}$  & A variable that assigns weights to the preserved columns in  increasing order.  For instance, $q_{i3}=0$ because $p_{i3}=0$, indicating that no weight is assigned to the third column as its corresponding vertex $v_3$ needs to be deleted. 
Whereas, $q_{i4}=3B$ signifies that there are three non-zero (preserved) columns until $k=4$ in the matrix $Q=(q_{ik})$. Similarly, $q_{i1}=B$, $q_{i2}=2B$, $q_{i5}=4B$, $q_{i6}=5B$, $q_{i7}=6B$ for $i \in \{1, \dots, n\}$. \\
$r^{j}_{ik}$  & This variable records the forward shift on each non-deleted column in the resulting adjacency matrix. Since at most $d$ original (non-padded) columns can be removed from the matrix $W=(w_{ik})$, the maximum shift experienced by any column is $d$. 
In this case $r^{1}_{i1}=r^{1}_{i2}=1$ because there is $j-1=0$ shift in first and second columns, as no column needs to be removed before first and second columns. Whereas, $r^{2}_{i3}=r^{2}_{i4}=r^{2}_{i5}=1$ because the remaining columns have a shift of $j-1=1$. All other values are $0$.
 \\
$s^{j}_{ik}$  & This variable shows that a given position $(i,k)$ in the resulting adjacency matrix can take the value $w_{i(k+j-1)}$, for some $j$, depending on the number of zero columns preceding the $q_{ik}$. For  example $s^{1}_{i1}=w_{i(1+1-1)} =[0, 1, 0, 1, B]$, $s^{1}_{i2}=w_{i(2+1-1)}=[1, 0, 1, 1, B]$, $s^{2}_{i3}=w_{i(3+2-1)}=[0, 1, 0, 1, B]$,  $s^{2}_{i4}=w_{i(4+2-1)}=[1, 1, 1, 0, B]$, and $s^{2}_{i5}=w_{i(5+2-1)}=[B, B, B, B, B]$. All other values of $s^{j}_{ik}$ are $0$.\\
$v'_{ik}$  & This variable specifies the positions where  $s^{j}_{ik}$ takes non-zero values for fixed $i$ and $k$ and gives the required output matrix $V'$ as shown in Fig.~\ref{fig:Del_N}. For example $v'_{i1}=s^{1}_{i1}=[0, 1, 0, 1, B]$, $v'_{i2}=s^{1}_{i2}=[1, 0, 1, 1, B]$, $v'_{i3}=s^{2}_{i3}=[0, 1, 0, 1, B]$, $v'_{i4}=s^{2}_{i4}=[1, 1, 1, 0, B]$, and $v'_{i5}=s^{2}_{i5}=[B, B, B, B, B]$ . \\
\end{longtable}
\end{example}
\section{GI$_d$-generative ReLU}\label{sec:GI}
For a vertex-labeled graph $G$ with $n$ vertices, a label set $\Sigma = \{1, 2, \ldots, m\}$ with an arbitrary vertex sequence $v_1, v_2, \ldots, v_n$, 
and a non-negative integer $d$, 
we define a {\em GI$_d$-generative ReLU} to be a 
ReLU neural network that generates any graph over $\Sigma$ whose graph edit distance is at most $d$ from $G$ due to the insertion of edges and/or vertices indicated by a random sequence $x=x_{1}, x_{2}, \ldots, x_{3d}$ such that $x_{j} \in \{1,2,  \ldots, n+d-1\}$, $1 \leq j \leq 2d$ and $x_{2d+j}\in \Sigma$, $1 \leq j \leq d$. 
The network inserts 
 \begin{enumerate}
\item [-]a vertex with label $x_{2d+j}$ corresponding to $x_j$ whenever $x_j = x_{j+d}$, and 
\item[-] an edge between $x_j$ and $x_{j+d}$ when $x_j \neq x_{j+d}$, \comb{excluding the invalid inputs given in Table~\ref{tab:reset}.}
 \end{enumerate}
To ensure a fixed number of nodes in the output, we pad the  label column $L$ with $d$ entries of $B$s, and extend the adjacency matrix $A$ by $d$ rows and $d$ columns with all $B$ entries, where $B \gg m$. 
The resulting padded matrices are denoted by $U=(u_{i})$ and $V=(v_{ik})$.
Observe that the total number $d'$ of vertex insertions is equal to the number of $j$s such that $x_j = x_{j+d}$. 
Therefore the output label column $U'=(u'_i)$ has $d-d'$ $B$s and the adjacency matrix $V'=(v'_{ik})$ has $d-d'$ rows and  $d-d'$ columns with all $B$s. 
Finally, by removing all $B$s, we can get the label column $L'=(\ell'_i)$ and the adjacency matrix $A'=(a'_{ik})$ of the resultant graph $G'$ as shown in Fig.~\ref{fig:Del_N}. 
Due to the random nature of $x$, some inputs may turn out to be invalid, and therefore require refinements to perform insertion operations. 
Such invalid inputs and their refinements are listed in Table~\ref{tab:reset}. 
\begin{table}[H]
\setlength{\tabcolsep}{0pt}
\centering
\begin{tabular}{|c|c|c|}\hline
Sr. no. & Invalid Inputs & Refinements \\\hline
(i)& $x_j \neq  x_{j+d}, x_j= x_k, x_{j+d}=x_{k+d}$, $k< j$ &  $x_j:=0$  \\
(ii)& $x_j \neq  x_{j+d}, x_j>n+d' $ &  $x_j:=0$ \\
(iii)& $x_j \neq  x_{j+d}, x_{j+d}>n+d' $ &  $x_{j+d}:=0$ \\
\hline
\end{tabular}
\caption{ Invalid inputs and their refinements.}
\label{tab:reset}
\end{table}
\comblu{To ignore the labels $x_{j+2d}$ corresponding to the valid edge insertions, set $x_{j+2d}:=B$  whenever $x_j \neq  x_{j+d}$ holds.  }
We discuss the existence of such networks in Theorem~\ref{thm:Ginn}. 
\begin{theorem}
\label{thm:Ginn}
For a vertex-labeled graph $G$ with $n$ vertices from 
$\Sigma = \{1, 2, \ldots, m\}$, and a non-negative integer $d$, 
there exists a GI$_d$-generative ${\rm ReLU}$ network with size $\mathcal{O}(n^2d)$ and constant depth.
\end{theorem}
\begin{proof}
Consider two vertex-labeled graphs $G$ and $G'$ of order $n$ over $\Sigma$, such that the graph $G'$ can be obtained from $G$ by using the insertion of vertices and/or edges indicated by a sequence $x_1, x_2, \ldots, x_{3d}$. 
We claim that the construction of $G'$ can be simulated by the following system of equations, where $j\in \{1, 2, \ldots, d\}$, $i,k \in \{1,2, \ldots, n+d\}$ unless stated otherwise. The constants $B, C$ are chosen to be large, with $C \gg B\gg \max(m,n)$.\bigskip\\
{Refinements of the input $x$:} 
The refinement (i) is carried out by  Eqs.~(\ref{eqe2}), (\ref{eqe'2}); refinement (ii) is performed by Eqs.~(\ref{eqf3}), (\ref{eqx13}); refinement (iii) is performed by Eqs.~(\ref{eqf'3}), (\ref{eqx23}); and refinement (iv) is performed by  Eq.~(\ref{eqg3}). 
Eqs.~(\ref{eqg'2}) and (\ref{eqx32}) are applied to arrange $x_{j+2d}$ in ascending order, ensuring that $B$ appears as the last entry of the output label column to ignore them easily. 
\begin{align}
e_{jk} &= \max \Big( \big(1-\delta(x_j, x_{j+d}) \big) \land  \delta(x_j, x_{k})\land \delta(x_{j+d}, x_{k+d}) , 0 \Big), \label{eqe2}\\
e'_j &= \max \Big( x_j - C \cdot \sum_{k=1}^{j-1} e_{jk} , 0 \Big), \label{eqe'2}\\
f_j &= \max  \Big( \big(1-\delta(e'_j, x_{j+d}) \big) \land H(e'_j-n-d'-1) , 0 \Big), \label{eqf3}\\
x^1_j &= \max \Big( e'_j - C f_j , 0  \Big), \label{eqx13}\\
f'_j  &= \max \Big( \big( 1-\delta(x_j, x_{j+d}) \big) \land H(x_{j+d}-n-d'-1) , 0 \Big), \label{eqf'3}\\
x^2_j &= \max \Big( x_{j+d} - C f'_j , 0 \Big), \label{eqx23}\\
g_j &= \max \Big( x_{j+2d} - C(1-\delta(x_j, x_{j+d})), 0 \Big) 
+ \max \Big( B - C \, \delta(x_j, x_{j+d}), 0 \Big), \label{eqg3}\\
g'_{j} &= \sum_{k=1}^{d} H(g_j - g_{k}) - \sum_{k=j}^{d} \delta(g_j, g_k),  \label{eqg'2}\\
x^3_j &= \sum_{i=1}^{d} \max \Big( g_i - C(1-\delta(j, g'_i+1)), 0 \Big).   \label{eqx32}
\end{align}
\bigskip\\
{Vertex insertion: }
To insert $d'$ 
isolated vertices corresponding to $x_j$ such that $x_j=x_{j+d}$, replace the entries in  $U$ with $B$ by the labels $x^3_{j}$ by using the Eqs.~(\ref{eqh2}), (\ref{equ'2}) and set all $B$ entries of the $d'$ number of rows and columns in the padded adjacency matrix $V$ into $0$ by using Eqs.~(\ref{eqp2})-(\ref{eqr2}).
\begin{align}
h_i &= \sum_{j=1}^{d} 
   \max \Big( x^3_j - C \big(1 - \delta(i, n+j)\big), 0 \Big), \label{eqh2}\\
u'_i &= \max \Big( \ell_i - C\big(1 - H(n-i)\big), 0 \Big) 
     + \max \Big( h_i - C\big(1 - H(i-n-1)\big), 0 \Big), \label{equ'2}\\
p_{ik} &= \max \Big( H(n+d' - i) \land H(i-n-1) \land H(n+d' - k), 0 \Big), \label{eqp2}\\
p'_{ik} &= \max \Big( B - C (1- p_{ik}), 0 \Big), \label{eqp'2}\\
q_{ik} &= \max \Big( H(n+d' - k) \land H(k-n-1) \land H(n-i), 0 \Big), \label{eqq2}\\
q'_{ik} &= \max \Big( B - C (1- q_{ik}), 0 \Big), \label{eqq'2}\\
r_{ik} &= v_{ik} - \big( p'_{ik} + q'_{ik} \big). \label{eqr2}
\end{align}
\bigskip\\
{Edge insertion: }
In order to insert the edges corresponding to $x^1_j=x^2_{j}$ for $j \in \{1, 2, \ldots, d\}$, by replacing $0$ with $1$ in $(x^1_j, x^2_{j})$-th and $(x^2_j, x^1_{j})$-th entries of the matrix $R = (r_{ik})$ as follows. 
\begin{align}
s^{j}_{ik} &= \max \Big( (1-\delta(x^1_j, x^2_{j})) 
                        \land  \delta(x^1_j, i)
                        \land \delta(x^2_j, k) ,0 \Big)+ \nonumber\\
            &\quad   \max \Big( (1-\delta(x^1_j, x^2_{j})) 
                        \land \delta(x^1_j, k) 
                        \land \delta(x^2_j, i) ,0  \Big),
\label{eqs2}\\
s'_{ik} &=  \sum_{j=1}^{d} s^{j}_{ik}, \label{eqs'2} \\
v'_{ik} &= s'_{ik} +\max \Big( r_{ik} - C  s'_{ik}, 0 \Big). \label{eqv'2}
\end{align}
For a fixed $j$, $S^j = (s^{j}_{ik})$ is the matrix with entries $(x^1_j, x^2_{j})$-th and $(x^2_j, x^1_{j})$-th equal to $1$ when 
``$x^1_j=i$ and $x^2_j=k$'' or ``$x^1_j=k$ and $x^2_j=i$'', and 
$0$ elsewhere. 
$S' = (s'_{ik})$ is a binary matrix that keeps the sum of entries of $s^{j}_{ik}$ for fixed $i$ and $k$. 
The final matrix $V= (v'_{ik})$ is then obtained by applying the vertex and edge insertions specified by the input.

To obtain the required graph $G'$, we remove $B$ from $U' = (u'_{i})$ and eliminate from $V' = (v'_{ik})$ all rows and columns of $B$s. This yields the label and adjacency matrices $L'$ and $A'$ of $G'$, respectively. Furthermore, the maximum function, $\delta$ and threshold functions appearing in the preceding equations can be realized using the ReLU activation function, as demonstrated in~\cite[Propositions~1 and~2]{MT2024}. 
As a result, a ${\rm GI}_d$-generative ${\rm ReLU}$ network of size 
$\mathcal{O}(n^2d)$ and constant depth can be constructed.
\end{proof}
\begin{example}
\label{exa:ins}
Consider the graph $G$ given in Fig.~\ref{fig:G_r}, $d=3$ and 
$x=4, 3, 7, 6, 3, 2, 1, 5, 2$, 
\combl{where the first six ($2d$) entries indicate the indices of vertices for edge or vertex insertion, and the last three ($d$) entries are the labels of newly inserted vertices, as depicted in red in Fig.~\ref{fig:Ins_N}.
Edge (resp., vertex) insertion is performed when $x_j \neq x_{j+d}$ (resp., $x_j = x_{j+d}$).
If an edge insertion is detected, the corresponding label entry $x_{j+2d}$ is ignored by setting it to $B$.
Finally, the updated labels $x_{1+2d}, \ldots, x_{d+2d}$ are arranged in ascending order.
For example, in this case, $x_1 = 4 \neq 6 = x_4$ (resp., $x_2 = 3 = x_5$) implies an edge insertion (resp., vertex insertion).
Therefore, $x_7$ is set to $B$.
The indices greater than $n = 5$ are set to zero; for example, $x_3 = 7$ is set to zero, and hence $x_{3+d}$ and $x_{3+2d}$ are ignored.
The updated $x$ is shown in Fig.~\ref{fig:Ins_N} as $x^1, x^2, x^3$, where the indices set to zero and the labels after rearrangement are depicted in red.
Vertex insertion is performed as follows.
The matrix $L$ is updated by inserting the label 5 of a new vertex, as depicted in red in $U'$, and by padding two $B$s to make the size $n+d = 5 + 3$, thereby ensuring a fixed-size output from the network.
Similarly, $A$ is updated to $R$ by adding a row and a column of zeros corresponding to the new vertex, as depicted in red in $R$.
The index of the new vertex is 6, and it is an isolated vertex at this stage; therefore, the corresponding entries in $R$ are 0.
Additionally, two rows and two columns, each filled with $B$s, are padded to maintain a fixed-size output.
Edge insertion is performed between vertices 4 and 6, since $x_1 = 4 \neq 6 = x_4$, by changing the corresponding entry from 0 (as marked in red in $R$) to 1 in $S'$, also depicted in red.
The entries of the padded rows and columns are set to zero.
Finally, the resultant matrix $V'$, incorporating both vertex and edge insertions, is obtained by adding $R$ and $S'$. } 
The resultant graph $G'$ is obtained by applying the insertion operations on $G$ due to the given $x$ is shown in Fig.~\ref{fig:Ins_N}.
We demonstrate the process of obtaining $L'$ and $A'$ of $G'$ by using Eqs.~(\ref{eqb0})-(\ref{eqy0}) as follows. 
\begin{figure}[H]
	\centering
	\includegraphics[scale = 0.52]{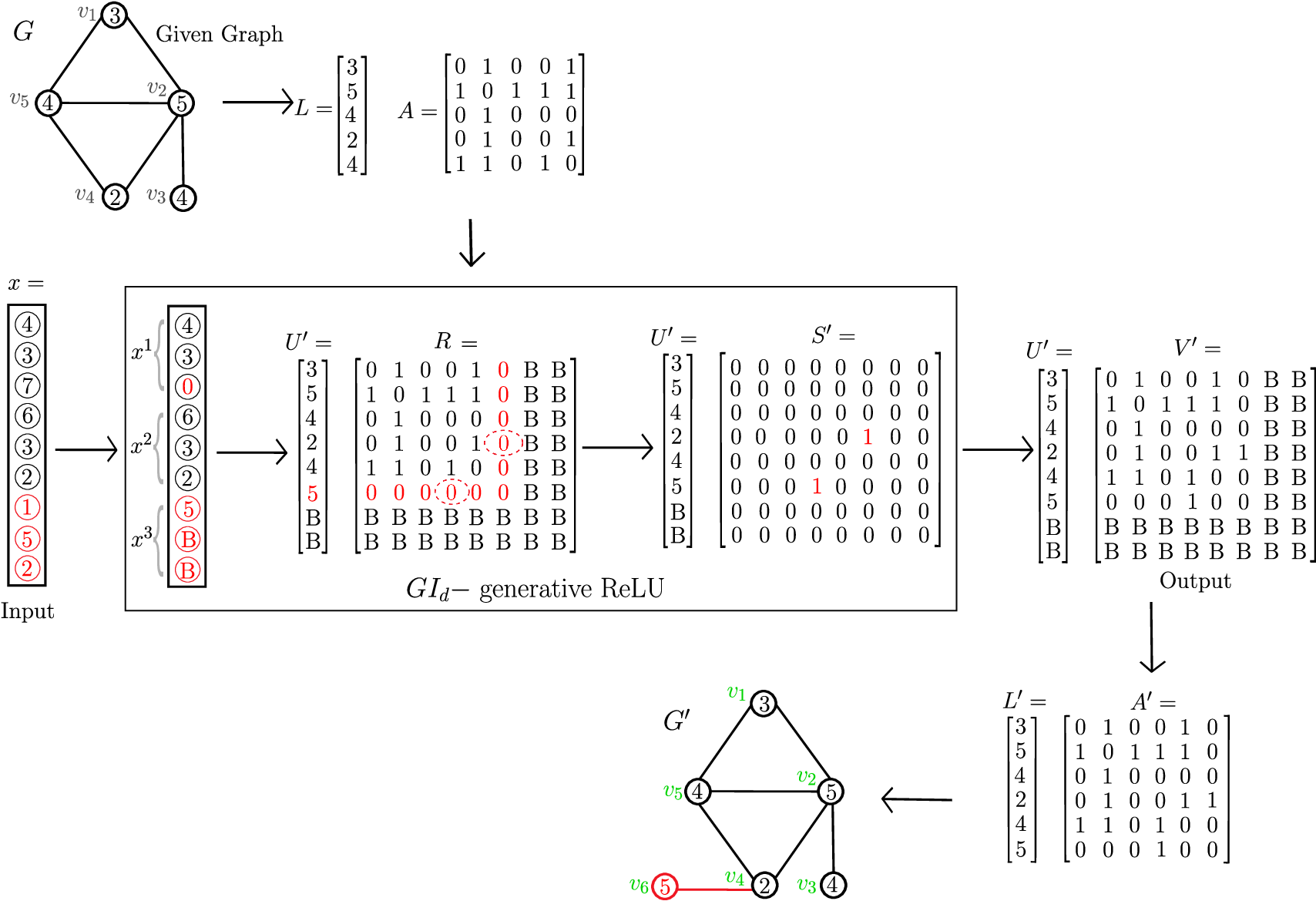}
	\caption{A demonstration of obtaining $G'$ from $G$ using a GI$_d$-generative ReLU, where $d=3$ and the input layer $x=4, 3, 7, 6, 3, 2, 5, 1, 2$.
}\label{fig:Ins_N}
\end{figure}
\begin{longtable}{c p{14cm}}
$x_j$ & Indicates vertex indices (resp.,  label of new vertices) when $1 \leq j \leq 2d$ (resp., $2d+1 \leq j \leq 3d$).\\
$e_{jk}$ & An indicator that takes value $1$ whenever there exists an index $k< j$ with $x_j= x_k$ and $x_{j+d}=x_{k+d}$, provided $x_j \neq  x_{j+d}$ for $j, k \in \{1, \dots , d\}$ to avoid invalid inputs. In this example, $e_{jk}=0$ for every $j,k$.\\
$e'_{j}$  &  Nullifies the repeated input $x_j$ if for any $k< j$, $e_{jk}=1$. In this case $e'_{j}=x_{j}$ for $j \in \{1, \dots , d\}$. \\
$f_{j}$  & A binary variable to identify if $e'_j$ is greater than $n+d'$. In this example, $f_{3}=1$ whereas $f_{1}=f_{2}=0$.\\
$x^1_{j}$  & Sets $e'_{j}$ to $0$ if it is greater than $n+d'$. That is, $x^1_{j}=0$ if $f_{j}=1$ and $x^1_{j}=e'_{j}$ otherwise. \\
$f'_{j}$  & Identifies whether $x_{j+d}$, if it is greater than $n+d'$. Here, $f'_{j}=0$ for all $j \in \{1, \dots , d\}$.\\
$x^2_{j}$  & Nullifies $x_{j+d}$ if it is greater than $n+d'$. That is, $x^2_{j}=0$ if $f'_{j}=1$ and $x^2_{j}=x_{j+d}$ otherwise. In this case $x^2_{j}=x_{j+d}$  for all $j \in \{1, \dots , d\}$. \\
$g_{j}$  & Sets $x_{j+2d}=B$, if $x_j \neq  x_{j+d}$ for a fixed $j \in \{1, \dots , d\}$. In this example $g_{2}=x_{8}=5$ and $g_{1}=g_{3}=B$. \\
$g'_{j}$  & Assigns a number from $0$ to $d-1$ to each value of $g_{j}$ to arrange them in ascending order. Here, $g'_{1}=1$, $g'_{2}=0$ and \combl{$g'_{3}=2$}. \\
$x^3_{j}$  & Arrange $g'_{j}$ in ascending order. 
In this case, $x^3_{1}=5$, $x^3_{2}=B$ and $x^3_{3}=B$.\\
$h_{i}$  & A variable that stores the labels to be inserted, i.e.,  $h_i =0$ for $i \le n$ and  $h_i = x^3_{j}$ for $i \geq {n+j}$. Here, $h=[0,0,0,0,0,5,B,B]$.\\
$u'_i$ & Entries of the resultant label column. The first $n$ entries are from the label column $L$, next $d'$ entries are the labels of the newly inserted vertices and last $d-d'$ entries are $B$. Here, $U' = (u'_i)=[3, 5, 4, 2, 4, 5, B, B]^T$.\\
$p_{ik}$ & Entries of a binary matrix which is $1$ if $i \in [n+1, n+d']$ and $k \le {n+d'}$ to indicate the rows corresponding to vertices in the padded adjacency matrix. In this case, $p_{6k}=1$ for $k \le 6$. All other values are $0$.\\
$p'_{ik}$ & Entries of a matrix which are $B$ when $p_{ik}=1$ and $0$ otherwise. Here $p'_{6k}=B$ for $k \le 6$.\\
$q_{ik}$ & Entries of a binary matrix which are $1$ if $k \in [n+1, n+d']$ and $i \le {n}$ to indicate the columns corresponding to new vertices in the padded adjacency matrix. Here, $q_{i6}=1$ for $i \le 5$. All other values are $0$.\\
$q'_{ik}$ & Entries of a matrix which are $B$ for $i$ and $k$ such that $q_{ik}=1$ and $0$ otherwise. In this example, $q'_{i6}=B$ for $i \le 5$.\\
$r_{ik}$ & A matrix with entries $0$ if $i\in [n+1, n+d']$, $k \le {n+d'}$ and for $k\in [n+1, n+d']$, $i \le {n}$. All other entries are same as padded matrix $V$. Here, $r_{i6}=0$ for \combl{$i \le 5$ and $r_{6k}=0$ for $k \le 6$}. The matrix $R=(r_{ik})$ is shown in Fig.~\ref{fig:Ins_N}.\\
$s^{j}_{ik}$ & For a fixed $j$,  $s^{j}_{ik}=1$ when  $x^1_j \neq  x^2_{j}$,  $i=x^1_j$ and $k=x^2_j$, or when  $i=x^2_j$ and $k=x^1_j$; otherwise $s^{j}_{ik}=0$. In this example only $s^{1}_{6,4}=s^{1}_{4,6}=1$.\\
$s'_{ik}$ & A binary variable which is $1$ whenever $s^{j}_{ik}=0$ and $0$ otherwise. Here, $s'_{6,4}=s'_{4,6}=1$ in the matrix $S'=(s'_{ik})$ as shown in Fig.~\ref{fig:Ins_N}\\
$v'_{ik}$  &Entry of the final padded adjacency matrix that is $1$ whenever $s'_{ik}=1$. All other entries are same as the padded adjacency $U$. In this example $v'_{6,4}=v'_{4,6}=1$ that represents newly inserted edge in the given graph $G$, as shown in Fig.~\ref{fig:Ins_N}.
\end{longtable}
\end{example}
\section{GE$_d$-generative ReLU}\label{sec:GED}
For a vertex-labeled graph $G$ with $n$ vertices, a label set $\Sigma = \{1, 2, \ldots, m\}$ with an arbitrary vertex sequence $v_1, v_2, \ldots, v_n$, 
and a non-negative integer $d$, we define a {\em GE$_d$-generative ReLU} to be a 
ReLU neural network that generates any graph over $\Sigma$ whose graph edit distance is at most $d$ from $G$ due to substitution of vertices and deletion and insertion of edges and/or vertices indicated by a random sequence $x=x_{1}, x_{2}, \ldots, x_{7d}$,
where each $x_j \in [0,1)$ has the form $x_j=i\cdot \Delta$, with $i$ an integer and 
$\Delta \leq 1$ a small positive constant. 
The sub-sequence $x_1, x_2, \ldots, x_{2d}$ (resp., $x_{2d+1}, x_{2d+2}, \ldots, x_{5d}$ and 
$x_{5d+1}, x_{5d+2}, \ldots,  x_{7d}$) corresponds to the substitution (resp., insertion and deletion) operations.
\combl{The conditions for the substitution, vertex/edge deletion, and vertex/edge insertion are explained in Sections~\ref{sec:GS}, \ref{sec:GD}, and~\ref{sec:GI}, respectively.} 
As a preprocessing step, the random inputs 
$x_j$ are converted into integers considering their corresponding operations as follows
\begin{enumerate}
\item[-] $x_j$ for $1 \leq j \leq d$ or $5d+1 \leq j \leq 7d$  indicates indices of vertices for substitution and deletion operations, and are converted into integers in 
$ \{0, \ldots, n\}$ if $x_j \in \big((i-1)/n, i/n\big]$, 
\item[-] $x_j$ for \combl{$2d+1 \leq j \leq 4d$}, indicates indices of vertices for insertion, and are converted into integers in $\{0, \ldots, n+d-1\}$ if $x_{j} \in \big((i -1)/{n+d-1},
 i/{n+d-1}\big]$,
\item[-] $x_j$ for $d+1 \leq j \leq 2d$ or $4d+1 \leq j \leq 5d$ indicates labels for substitution and insertion, and are converted into integers in $\Sigma$ if 
$x_{j} \in \big((i -1)/m, i/m\big]$, for $i = 1$, 
$x_j \in \big[(i -1)/m, i/m\big]$ otherwise. 
\end{enumerate}
For example, when $n = 5$, $d=3$ and $m  = 10$, the conversion scheme is given in Table~\ref{tab:conv}. 
To output a fixed number of nodes, we assume that the label matrix $L$ (resp., adjacency matrix $A$) of a given graph $G$ is padded with $2d$ $B$s (resp., $2d$ of rows and column with all $B$ entries) to get  the matrices $U$ (resp., $V$), where $B \gg \max(m,n)$.
The required label and adjacency matrices, $L'$ and $A'$ resp., can be obtained by removing $B$ from $U'$ and $V'$. 
Moreover, if $d_1$, $d_2$ and $d_3$ are the number of operations performed for substitution, insertion and deletion resp., therefore $d_1 + d_2 + d_3 \leq d$.
\begin{table}[h!]
\setlength{\tabcolsep}{0pt}
\centering
\begin{tabular}{|c|c|c|}\hline
\makecell{~For positions $x_j$ with ~\\$1 \leq j \leq d~$~\\ and $5d+1 \leq j \leq 7d$}& \makecell{~For positions $x_j$ with ~\\$2d+1 \leq j \leq 4d$~}& \makecell{~For values $x_j$ with~\\ $d+1 \leq j \leq 2d$ and ~\\$4d+1 \leq j \leq 5d$~}\\\hline
$~(-1/5, 0] \rightarrow 0$ & ~~$(-1/7, 0] \rightarrow 0$  & ~~~~~~~$[0, 1/10] \rightarrow 1$\\
$~~~(0, 1/5] \rightarrow 1$ & ~~~~~$(0, 1/7] \rightarrow 1$ & ~~$(1/10, 2/10] \rightarrow 2$\\
$~~(1/5, 2/5] \rightarrow 2~~$ & ~~$(1/7, 2/7] \rightarrow 2$   & ~~$(2/10, 3/10] \rightarrow 3$\\
$~~(2/5, 3/5] \rightarrow 3~~$ & ~~$(2/7, 3/7] \rightarrow 3$  & ~~$(3/10, 4/10] \rightarrow 4$\\
$~~(3/5, 4/5] \rightarrow 4~~$ & ~~$(3/7, 4/7] \rightarrow 4$ & ~~$(4/10, 5/10] \rightarrow 5$\\
$~~(4/5, 5/5] \rightarrow 5~~$ & ~~$(4/7, 5/7] \rightarrow 5$  & ~~$(5/10, 6/10] \rightarrow 6$ \\
 & ~~$(5/7, 6/7] \rightarrow 6$ & ~~$(6/10, 7/10] \rightarrow 7$\\
 & ~~$(6/7, 7/7] \rightarrow 7$ & ~~$(7/10, 8/10] \rightarrow 8$ \\
& ~~& ~~$(8/10, 9/10] \rightarrow 9$\\
& ~~& ~~~~$(9/10, 10/10] \rightarrow 10~~$ \\\hline
\end{tabular}
\caption{Conversion table from real numbers to integers.}
\label{tab:conv}
\end{table}
The existence of GE$_d$-generative ${\rm ReLU}$ networks is discussed in Theorem~\ref{thm:Enn}. 
\begin{theorem}
\label{thm:Enn}
Given a vertex-labeled graph $G$ with $n$ vertices over the alphabet set
$\Sigma = \{1, 2, \ldots, m\}$, and a non-negative integer $d$, 
there exists a GE$_d$-generative ${\rm ReLU}$ network with size $\mathcal{O}(n^2d)$ and constant depth.
\end{theorem}
\begin{proof}
Let $G$ and $G'$ be two vertex-labeled graphs with $n$ vertices such that
$G'$ can be constructed from $G$ through the graph edit operations (substitution, insertion and deletion) indicated by a sequence $x_1, x_2, \ldots, x_{7d}$. 
We claim that the process of obtaining $G'$ from $G$ can be simulated with the following system of equations, where $B$ and $C$ are large numbers with $C \gg B \gg \max(m,n)$.\bigskip\\
{Conversion of input into integers:}
First, convert the input $x_j$ into integers by using the following equations. 
\begin{align} 
p_{i}^{j} &= \left[ (i-1)/n \leq x_{j} \leq i/ n \right] -  \delta(x_j, (i-1)/n),  
 \nonumber\\
&~~~~ \text{~for~} i \in \{0, 1, \ldots, n\}, j \in \{1, \ldots, d, 5d+1, \ldots, 7d\}, \label{eqp5}\\
q_{i}^{j} &= \left[ (i-1)/{n+d-1} \leq x_{j} \leq i/{n+d-1} \right] -  \delta(x_j, (i-1)/{n+d-1}),  
 \nonumber\\
&~~~~ \text{~for~} i \in \{0, 1, \ldots, {(n+d-1)}\}, j \in \{2d+1, \ldots, 4d \}, \label{eqq5}\\
r_{i}^{j} &= 
\begin{cases}
\left[ (i-1)/m \leq x_{j} \leq {i}/ m \right] ~~~~~~~~~\text{~if~}  {i} =1, \\[5pt]
\left[ (i-1)/m \leq x_{j} \leq {i}/ m \right] -   ~~~~~\text{~if~}  
{i} \in \{2, \ldots, m\},  \\
~~\delta(x_j, (i-1)/m)  , \label{eqr5}
\end{cases}
 \nonumber\\
&~~~~ \text{~for~}  j \in \{d+1, \ldots, 2d, 4d+1, \ldots, 5d\},\\
p'_{j} &= \sum^{n}_{i=0} p_{i}^{j} \cdot i  \text{~for~} 
 j \in \{1, \ldots, d, 5d+1, \ldots, 7d\}, \label{eqp'5}\\
q'_{j} &= \sum^{n+d-1}_{i=0} q_{i}^{j} \cdot i  \text{~for~} 
 j \in \{2d+1, \ldots, 4d\}, \label{eqq'5}\\
r'_{j} &=\sum^{m}_{i=1}  r_{i}^{j} \cdot {i}  \text{~for~}  j \in \{d+1, \ldots, 2d, 4d+1, \ldots, 5d\}. \label{eqr'5}
\end{align}
The variables $p_{i}^{j}$, $q_{i}^{j}$ and $r_{i}^{j}$ indicate whether the $j$-th input lies in the $i$-th interval, for $i \in \{0, \ldots, n\}$, $i \in \{0, \ldots, {n+d-1}\}$ and $i\in \{1, \ldots, m\}$, respectively. The variables $p'_j $, $q'_j $ and $r'_j$ represent the integer values corresponding to $x_j$ within their respective intervals. Therefore, the input in integer form is given as:
\begin{align}
x'_{j} &= 
\begin{cases}
p'_j  ~~~\text{~for~}  j \in \{1, \ldots, d\}, \\[5pt]
r'_j  ~~~\text{~for~}  j \in \{d+1, \ldots, 2d\}, \\[5pt]
q'_j ~~~\text{~for~}  j \in \{2d+1, \ldots, 4d\},\\[5pt]
r'_j ~~~\text{~for~}  j \in \{4d+1, \ldots, 5d\},\\[5pt] 
p'_j ~~~\text{~for~}  j \in \{5d+1, \ldots, 7d\}. \label{eqp5}
\end{cases}
\end{align}
\bigskip\\
{Elimination of invalid inputs: }
To avoid redundant substitution, insertion, and deletion operations, we first eliminate repeated inputs.
Since $x'_j$, $1 \leq j \leq 2d$, is responsible for substitution operations, the repetition from $x'_j$, $1 \leq j \leq d$ is removed by using  $e_j$ of Eq.~(\ref{eqe1}),
and variables $e_{jk}$ and $e'_j$ of Eqs.~(\ref{eqe2}), ~(\ref{eqe'2}) are used to remove the repetition from $x'_j$, $2d+1 \leq j \leq 3d$, the sequence that is used in the insertion operation. To handle the redundant deletion operation, repetition from $x'_j$, $5d+1 \leq j \leq 6d$ is removed by using the following equations.
\begin{align}
s_{jk} &= \max \Big( \delta(x'_{j}, x'_{k}) \land \delta(x'_{j+d}, x'_{k+d}), 0  \Big)  \text{~for~}  j \in \{5d+1, \ldots, 6d\}.   \label{eqs5}\\
s'_{j} &= \max \Big( x'_{j} -  C \cdot \sum_{k=1}^{j-1} s_{jk}, 0  \Big)  \text{~for~}  j \in \{5d+1, \ldots, 6d\}.   \label{eqs'5}
\end{align}
Let $x''$ denote the resultant sequence such that:
\begin{align}
x''_{j} &= 
\begin{cases}
e_j  ~~~\text{~for~}  j \in \{1, \ldots, d\}, \\[5pt]
x'_j  ~~~\text{~for~}  j \in \{d+1, \ldots, 2d\}, \\[5pt]
e'_j ~~~\text{~for~}  j \in \{2d+1, \ldots, 3d\},\\[5pt]
x'_j  ~~~\text{~for~}  j \in \{3d+1, \ldots, 5d\}, \\[5pt]
s'_j ~~~\text{~for~}  j \in \{5d+1, \ldots, 6d\}, \\[5pt]
x'_j ~~~\text{~for~}  j \in \{6d+1, \ldots, 7d\}.  \label{eqx''5}
\end{cases}
\end{align}
\bigskip\\
{Removal of extra edit operations: }
Since $d_1 + d_2 + d_3$ should not exceeds $d$, there may exist some excess edit operations.
To find the number of such operations, the following equations are used.   
\begin{align}
t_{j} &= 
\begin{cases}
\max \Big(1- \delta(x''_j, 0), 0 \Big)   ~~~\text{~for~}  j \in \{1, \ldots, d\}, \\[10pt]
\max \Big(1- \big(\delta(x''_j, 0) + \delta(x''_{j+d}, 0) \big), 0 \Big)   ~~~\text{~for~}  j \in \{2d+1, \ldots, 3d\},\\[10pt]
\max \Big(1- \big(\delta(x''_j, 0) + \delta(x''_{j+d}, 0) \big), 0 \Big)   ~~~\text{~for~}  j \in \{5d+1, \ldots, 6d\}  \label{eqt5},
\end{cases}\\[10pt]
t'_{j} &=
\left[ \sum^{j}_{k=1} t_{k} \geq d+1 \right]  \text{~for ~}  j \in \{1, \ldots, d, 2d+1, \ldots, 3d, 5d+1, \ldots, 6d\}. \label{eqt'5}
\end{align}
The excess positions can be removed by assigning them value 0 as follows: 
\begin{align}
w_{j} &= \max (x''_j - C \cdot t'_{j} , 0)  \nonumber\\
&~~~~  \text{~for~}  j \in \{1, \ldots, d, 2d+1, \ldots, 3d, 5d+1, \ldots, 6d\}. \label{eqw5}
\end{align}
Because in the insertion portion, $x_j=x_{j+d}=0$ which results in  $\delta(x_j, x_{j+d})=1$. To handle this issue, replace $0$ with $B$ in $w_j$ for $2d+1 \leq j \leq 3d$ using the following equations:
\begin{align}
w'_j &= \max \Big( w_{j} - C\cdot \delta(w_{j}, 0), 0 \Big) 
+ \max \Big( B -C (1-\delta(w_{j}, 0)), 0 \Big). \label{eqw'5}
\end{align}
The preprocessed input $X_j$ for edit operations is finally obtained as follows:
\begin{align}
X_{j} &= 
\begin{cases}
w_j  ~~~\text{~for~}  j \in \{1, \ldots, d\}, \\[2pt]
x''_j  ~~~\text{~for~}  j \in \{d+1, \ldots, 2d\}, \\[2pt]
w'_j ~~~\text{~for~}  j \in \{2d+1, \ldots, 3d\},\\[2pt]
x''_j ~~~\text{~for~}  j \in \{3d+1, \ldots, 5d\},\\[2pt] 
w_j ~~~\text{~for~}  j \in \{5d+1, \ldots, 6d\},\\[2pt]
x''_j ~~~\text{~for~}  j \in \{6d+1, \ldots, 7d\}. \label{eqX5}
\end{cases}
\end{align}
\bigskip\\
{Application of edit operations:} 
Apply substitution operations on padded label and adjacency matrices $U$ and $V$ resp,. by using Eqs.~(\ref{eqf1})-(\ref{eql1}) of Theorem~\ref{thm:Gsnn} and  $X_{j}$, $j \in \{1, \ldots, 2d\}$ as inputs to get the matrices $U^1$ and $V^1$. 
Apply insertion operations on $U^1$ and $V^1$ by using 
Eqs.~(\ref{eqg3})-~(\ref{eqv'2}) of Theorem~\ref{thm:Ginn}, with $X_{j}$, $j \in \{2d+1, \ldots, 5d\}$, to get $U^2$ and $V^2$. 
Apply deletion operations on $U^2$ and $V^2$ according to Theorem~\ref{thm:Gdnn} and $X_{j}$, $j \in \{5d+1, \ldots, 7d\}$ 
to get $U'$ and $V'$.
Finally, the label and adjacency matrices $L'$ and $A'$ of the required graph $G'$ can be obtained by eliminating $B$s from $U'$ and $V'$.

All equations utilize the maximum function, $\delta$ or $[a \geq \theta]$ function, and therefore, by  Theorems~\ref{thm:Gdnn},~\ref{thm:Gsnn} and~\ref{thm:Ginn}, there exists a $GE_d$-generative ${\rm ReLU}$ network of size $\mathcal{O}(n^2 d)$ and constant depth.
\end{proof}
\begin{example}
\label{exa:Uni}
Reconsider the graph $G$ shown in Fig.~\ref{fig:G_r}, $d = 3$, $m = 10$, and input 
$x=0.45, 0, 0.59, 0, 0.4, 0.15, 0.11, 0.05, 0.88$, $0.55, 0.44, 0.93, 0.52$, 0.87, 0.03, 0.33, 0.4, 0, 0.79, 0.65, 0.9, 
\combl{where the first $2d$ entries are used for substitution, the next $3d$ entries are used for insertion, and the final $2d$ entries are used for deletion operations.
In the first step, decimals are converted into integers using the conversion Table~\ref{tab:conv} to obtain the vector $x' = [3, 0, 3, 1, 4, 2, 1, 1, 7, 4, 4, 0, 6, 9, 1, 2, 2, 5, 4, 4, 5]$.
For example, the index entry (resp., value/label entry) $x_1 = 0.45$ (resp., $x_4 = 0$) belongs to $(2/5, 3/5]$ (resp., $[0, 1/10]$), and therefore $x'_1 = 3$ (resp., $x'_4 = 1$).
In the second step, repetitions are removed from the indices for substitution, insertion, and deletion to obtain $x'' = [3, 0, 0, 1, 4, 2, 1, 0, 7, 4, 4, 0, 6, 9, 1, 2, 0, 5, 4, 4, 5]$.
For example, the indices $x'_1 = x'_3 = 3$ correspond to substitution; therefore, $x''_3 = 0$ to ignore it. Similarly, the index pairs $(x'_7, x'_{10}) = (x'_8, x'_{11}) = (1, 4)$ correspond to insertion; therefore, $x''_8 = 0$ to ignore it. Likewise, $(x'_{16}, x'_{19}) = (x'_{17}, x'_{20}) = (2, 4)$ correspond to deletion; therefore, $x''_{17} = 0$ to ignore it.
In the third step, edit operations exceeding $d$ are ignored to obtain $X = [3, 0, 0, 1, 4, 2, 1, B, 7, 4, 4, 0, 6, 9, 1, 2, 0, 0, 4, 4, 5]$.
For example, the indices $x''_1 = 3$, $(x''_7, x''_{10}) = (1, 4)$, and $(x''_{16}, x''_{19}) = (2, 4)$ are the only valid substitution, insertion, and deletion operations, respectively, until the fourth valid edit operation (a deletion) $(x''_{18}, x''_{21}) = (5, 5)$, the count of which exceeds $d = 3$; hence, $X_{18} = 0$ to ignore it.
Finally, zero indices corresponding to insertion operations are set to $B$, as the proposed network for insertion cannot handle zero indices. In this case, the index $x''_{8} = 0$ corresponds to insertion and is therefore 
set to $X_{8} = B$ so that it can be ignored. }
We illustrate the process of obtaining these vectors to get $G'$ by using Theorem~\ref{thm:Enn}. 
The meanings of Eqs.~(\ref{eqp5})-(\ref{eqX5}) are explained in Example~\ref{exa:Uni}, while the details of the deletion, substitution, and insertion operations are given in Examples~\ref{exa:Del},~\ref{exa:sub} and~\ref{exa:ins}.
\begin{longtable}{c p{14cm}}
${p}^j_i$ & This variable specifies the interval for the position $x_j $, where $ j \in \{1, \ldots, d, 5d+1, \ldots, 7d\}$. ${p}^j_i=1$ means that $x_j$ lies in the $i$-th interval, i.e., the interval 
$( (i-1)/n , i/n]$. In this case ${p}^{1}_3={p}^{2}_0={p}^{3}_3$~$={p}^{16}_2={p}^{17}_2={p}^{18}_5={p}^{19}_4={p}^{20}_4={p}^{21}_5=1$. All other values are zero. \\
${q}^j_{i}$  & Specifies the interval for each label $x_j$, where
 $j \in \{2d+1, \ldots, 4d\}$.  
 ${q}^j_{i}=1$ means that the $j$-th input lies in the ${i}$-th interval, i.e., $x_j$ lies in $[(i-1)/(n+d-1), i/(n+d-1)]$. 
 In this case, ${q}^{7}_1={q}^{8}_1={q}^{9}_7={q}^{10}_4={q}^{11}_4={q}^{12}_0=1$, and all other values are zero.\\
${r}^j_{i}$  & It specifies the interval for each label $x_j$, where
 $ j \in \{d+1, \ldots, 2d, 4d+1, \ldots, 5d\}$.  
 ${r}^j_{i}=1$ means that the $j$-th input lies in the ${i}$-th interval, i.e., $((i-1)/m, i/m]$. In this case, ${r}^{4}_1={r}^{5}_4={r}^{6}_2={r}^{13}_6={r}^{14}_9={r}^{15}_1=1$, and all other values are zero.\\
$p'_{j}$  & This variable assigns each position $x_j$ an integer $i$ if $x_j$ belongs to the $i$-th interval, i.e., if ${p}^j_i = 1$ then ${p'}_j = i$. Here, $p'_{1}=3$, $p'_{2}=0$, $p'_{3}=3$, $p'_{16}=2$, $p'_{17}=2$, $p'_{18}=5$, $p'_{19}=4$, $p'_{20}=4$, $p'_{21}=5$. \\
$q'_{j}$  & It assigns each position $x_j$ an integer $i$ if $x_j$ belongs to the $i$-th interval, i.e., if $q_i^j = 1$ then $q'_j=i$. In this example, $q'_{7}=1$, $q'_{8}=1$, $q'_{9}=7$, $q'_{10}=4$, $q'_{11}=4$, $q'_{12}=0$.\\
$r'_{j}$ & This variable assigns each label $x_j$ an integer $i$ if $x_j$ belongs to the $i$-th interval, i.e., if $r_i^j = 1$ then $r'_j=i$. In this example, $r'_{4}=1$, $r'_{5}=4$, $r'_{6}=2$, $r'_{13}=6$, $r'_{14}=9$, $r'_{15}=1$.\\
$x'_{j}$ & It combines the conversions of all parts of the input. Therefore, $x'=[3, 0, 3, 1, 4, 2, 1, 1, 7, 4, 4, 0, 6, 9, 1, 
2, 2, 5, 4, 4, 5]$.\\
$s_{jk}$  & To avoid the extra deletion operations, this variable indicates the indices in the deletion portion that are repeated. $s_{jk} = 1$ if $x'_j=x'_k$ and $x'_{j+d}=x'_{k+d}$. 
In this case $s_{16,16}=s_{16,17}=s_{17,17}=1$. For all other values $s_{jk}=0$.\\
$s'_{j}$  & It nullifies the effect of repeated index by converting it into $0$. In this case, $s'_{17}=0$. All other values are same as $x'_j$.\\
$x''_{j}$  &It combines the integer conversion of the input after the removal of repeated indices for each edit operation. 
In this case, $x''=[3, 0, 0, 1, 4, 2, 1, 0, 7, 4, 4, 0, 6, 9, 1, 
2, 0, 5, 4, 4, 5]$.\\
$t_{j}$  & This variable indicates the possible indices where the edit operation can be applied. Here, for substitution (resp., insertion and deletion) there is only one (resp., one and two) non-zero index on which edit operation can be applied, i.e., $x''_1=3$ (resp., $x''_7=1$ and $x''_{16}=2$, $x''_{18}=5$). Therefore,  $t_{1}=t_{7}=t_{16}=t_{18}=1$. All other values are $0$.\\
$t'_{j}$  & Since $d_1+d_2+d_3 \leq d$, this variable indicates the exceeded indices. Thus, $t'_{jk} = 1$ if sum of number of indices on which edit operation can be applied exceeds $d$.
In this case $t'_{18}=1$. For all other values $t'_{j}=0$.\\
$w_{j}$  &This variable nullifies the exceeded indices. Thus, $w_j=0$ if $t'_{j}=1$.
In this case $w_{18}=0$. All other values are the same as $x''_j$.\\
$w'_{j}$  & This variable is used to ignore the effect of $x_j=x_{j+d}=0$ which would otherwise result in  $\delta(x_j, x_{j+d})=1$ in the insertion case,  by replacing $0$ with $B$ in $w_{j}$ for $2d+1 \leq j \leq 3d$. Thus, $w'_{8} = B$. All other $j$, $w'_{j}=w_{j}$.\\
$X_{j}$  &This variable gives the final input for all three edit operations. Thus, $X=[3, 0, 0, 1, 4, 2, 1, B, 7, 4, 4, 0, 6, 9, 1, 
2, 0, 0, 4, 4, 5]$.
\end{longtable}
\end{example}
\section{Computational Results and Discussion} \label{sec:exp}
To evaluate the scalability of the proposed networks for generating graphs with the number of vertices $n$  and a desired graph edit distance $d$, we conducted a series of computational experiments by varying both parameters to generate graphs by performing insertion, deletion and substitution operations using the proposed ${\rm GE}_d$ network. 
The experiments were performed on a Linux-based server equipped with an Intel Xeon Gold 5222 CPU (3.80 GHz, 16 cores). The server also contains an NVIDIA A100 PCIe GPU with 40 GB of memory (CUDA 11.2), although the experiments reported in this study were conducted using the CPU only. 
In these experiments, the number of vertices $n$ ranged from 100 to 1400, while the desired graph edit distance $d$ ranged from 10 to 140. For each pair $(n, d)$, we measured the computational time in seconds required by the proposed networks to generate a graph satisfying the specified edit distance constraint. 
In our experimental design, we considered parameter combinations under the condition $d \leq n$. 
Consequently, the parameter pairs such as $d = 110$ with $n = 100$ were not evaluated. The experiments were conducted progressively by increasing both $n$ and $d$ in order to examine how the computational cost grows with the problem size. For each value of $d$, the number of vertices was increased until the computation encountered memory limitations. When the required memory exceeded the available resources, the corresponding experiment could not be completed, and larger instances for that configuration were not tested. For example, when $d = 120$ and $n = 300$, the computation resulted in a memory-out issue; therefore, larger instances for this parameter setting were not evaluated. 

The results are presented in Table~\ref{tab:time_table} and Fig.~\ref{fig:time_plot} which indicate that the computational time increases steadily with the number of vertices $n$. For example, when $d = 10$, the running time grows from 6 seconds for $n = 100$ to 1250 seconds for $n = 1400$. 
A similar trend is observed when the number of vertices is fixed and the edit distance increases. For instance, when $n = 200$, the running time rises from 19 seconds for $d = 10$ to 1560 seconds for $d = 140$, indicating that larger edit distances require substantially more computation.
Comparing these two observations shows that the increase in running time is more pronounced when $d$ grows while $n$ is fixed, whereas the growth is more gradual when $n$ increases for a fixed $d$. This suggests that the desired edit distance has a stronger influence on the computational effort, as larger values of $d$ significantly expand the space of feasible graph transformations that must be explored.
Overall, the experiments show that the proposed method can handle graphs with up to 1400 vertices for smaller edit distances. However, the computational cost increases considerably as $d$ grows, limiting the range of instances for larger edit distances. 
\begin{table}[ht]
\centering
\small
\setlength{\tabcolsep}{3.8pt}{
\begin{tabular}{|c|cccccccccccccc|}
\hline
$d \backslash n$ 
& 100 & 200 & 300 & 400 & 500 & 600 & 700 & 800 & 900 & 1000 & 1100 & 1200 & 1300 & 1400 \\
\hline
10  & 6 & 19 & 42 & 74 & 123 & 173 & 245 & 337 & 511 & 697 & 816 & 949 & 1090 & 1250 \\
20  & 14 & 42 & 87 & 146 & 239 & 333 & 444 & 633 & 946 & 1288 &  &  &  &  \\
30  & 27 & 73 & 141 & 229 & 376 & 527 & 668 & 953 &  &  &  &  &  &  \\
40  & 46 & 112 & 204 & 328 & 528 & 757 & 916 & 1298 &  &  &  &  &  &  \\
50  & 73 & 162 & 283 & 448 & 705 & 1013 & 1199 & 1739 &  &  &  &  &  &  \\
60  & 109 & 223 & 373 & 574 & 902 & 1258 &  &  &  &  &  &  &  &  \\
70  & 159 & 301 & 481 & 767 & 1134 &  &  &  &  &  &  &  &  &  \\
80  & 222 & 390 & 616 & 986 & 1416 &  &  &  &  &  &  &  &  &  \\
90  & 304 & 503 & 784 & 1225 &  &  &  &  &  &  &  &  &  &  \\
100 & 406 & 643 & 1022 &  &  &  &  &  &  &  &  &  &  &  \\
110 &  & 813 & 1261 &  &  &  &  &  &  &  &  &  &  &  \\
120 &  & 1009 &  &  &  &  &  &  &  &  &  &  &  &  \\
130 &  & 1249 &  &  &  &  &  &  &  &  &  &  &  &  \\
140 &  & 1560 &  &  &  &  &  &  &  &  &  &  &  &  \\
\hline
\end{tabular}
}
\caption{Computation time (sec.) for different values of $n$ and $d$ to generate 
desired graphs with the proposed network GE$_d$. }
\label{tab:time_table}
\end{table}
 \begin{figure*}[t!]
	\centering
	\includegraphics[scale = 0.55]{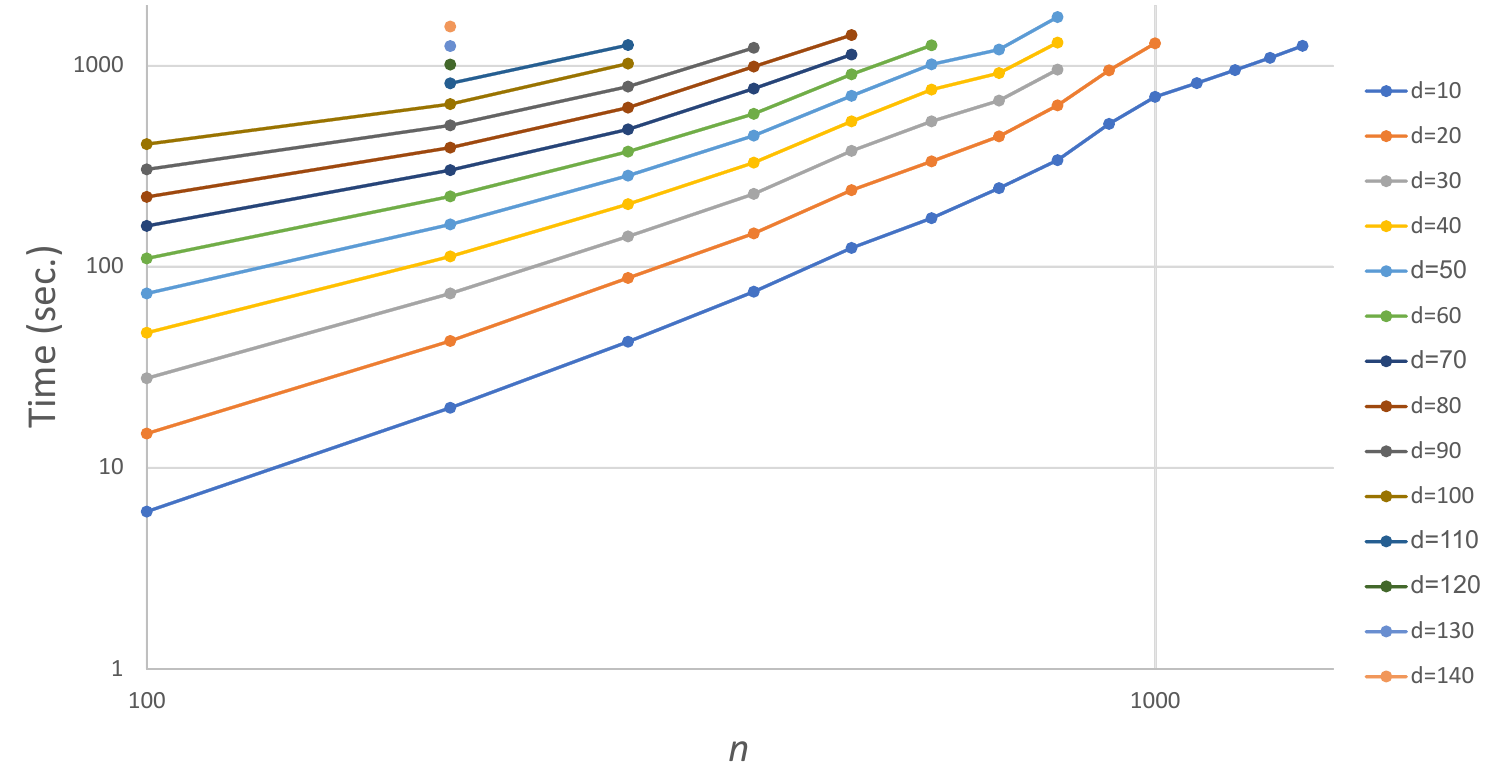}
	\caption{{Log-log plot between $n$ and computation time for different edit distance $d$.
 } }\label{fig:time_plot}
\end{figure*}

Additionally, we conducted a comparative analysis by generating graphs with a given edit distance using two state-of-the-art graph generative models, GraphRNN by You et al.~\cite{You2018GraphRNN} and GraphGDP by Huang et al.~\cite{Huang2022GraphGDP}. 
For this purpose, we selected six input graphs with the number of vertices $n = 10, 20, 30, 40, 50, 100$ and the number of edges $|E| = 9, 20, 130, 250, 918, 1536$, respectively. For each graph, we generated graphs with at most edit distance $d = 5, 10, 15, 20, 25, 50$, respectively, using the proposed network  GE$_d$  as well as GraphRNN and GraphGDP. 

To evaluate the quality of the generated graphs, we approximated the graph edit distance between the generated graphs and the corresponding input graph. For a consistent approximation, we considered unlabeled graphs and allowed only edge deletion and insertion operations when generating graphs with the proposed network GE$_d$. Under this setting, all generated graphs must have the same number of vertices as the input graph and the number of edges must lie within the range $[|E|-d, |E|+d]$, and graph edit distance at most $d$. 
\combl{Furthermore, the edit distance between the underlying graphs is the symmetric difference of the edges of the graphs~\cite{martin2013edit}}. 

For each triplet $(n, |E|, d)$, 500 unlabeled graphs were generated using the proposed network GE$_d$. 
\combl{The proposed GE$_d$ network requires no training, unlike GraphRNN and GraphGDP.}
GraphRNN (dependent Bernoulli variant) and GraphGDP were trained on graphs obtained by GE$_d$ using a 4-layer RNN and a 4-layer GNN, respectively, each with hidden neuron size 128. Training was performed for approximately 3000 epochs with batch size 32, using learning rates of 0.003 for GraphRNN and 0.00002 for GraphGDP. For both models, 80\% of the generated graphs were used for training and the remaining 20\% for testing, while the default settings were used for the other parameters.
After training, 500 graph samples were generated for each input graph using GraphRNN~\cite{You2018GraphRNN} and GraphGDP~\cite{Huang2022GraphGDP}. The graph edit distances between these generated graphs and the corresponding input graph were approximated. 
The number of edges and the graph edit distance of the graphs with 
$n$ vertices generated by these models are illustrated in Fig.~\ref{fig:valid_graph}. 
A summary of these results is given in Table~\ref{tab:comparison}, which reports the number $N_n$ of generated graphs with $n$ vertices, the number $N_{|E|}$ of graphs whose edge counts lie within the acceptable range $[|E|-d, |E|+d]$, and the number $N_d$ of graphs whose graph edit distance from the input graph is at most~$d$.

The results show that the proposed method GE$_d$ consistently satisfies all the required constraints. For every tested configuration, all 500 generated graphs preserve the number of vertices ($N_n = 500$), have edge counts within the acceptable range ($N_{|E|} = 500$), and achieve the desired graph edit distance ($N_d = 500$). This indicates that the proposed approach reliably generates graphs that strictly satisfy the structural and edit distance constraints imposed by the input graph.

In contrast, the graphs generated by GraphRNN frequently violate these constraints. Although GraphRNN occasionally produces graphs with acceptable edge counts or vertex numbers, none of the generated graphs achieve the required edit distance ($N_d = 0$ for all cases). Moreover, the number of valid graphs with respect to vertex and edge constraints decreases as the graph size increases. For example, when $n = 100$, only 438 out of 500 graphs preserve the correct number of vertices and none satisfy the edge constraint, indicating a significant degradation in structural consistency for larger graphs.

GraphGDP performs slightly better in preserving the number of vertices, with $N_n = 500$ for all configurations, and often produces graphs whose edge counts fall within the acceptable range. However, similar to GraphRNN, none of the generated graphs satisfy the required edit distance constraint ($N_d = 0$ across all cases). In particular, the number of graphs with acceptable edge counts decreases significantly for larger graphs; for instance, when $n = 100$, only 60 out of 500 generated graphs fall within the acceptable edge range.

Overall, the results demonstrate that the proposed GE$_d$ network is significantly more reliable for generating graphs under explicit structural constraints. Unlike the baseline generative models, the proposed method consistently produces valid graphs that simultaneously satisfy the vertex, edge, and edit distance requirements across all tested instances.

The input graphs and the resultant graphs obtained in these experiments are available at~\url{https://github.com/MGANN-KU/GraphGen\_ReLUNetworks}.
\begin{figure}[H]
	\centering
	\includegraphics[scale = 0.38]{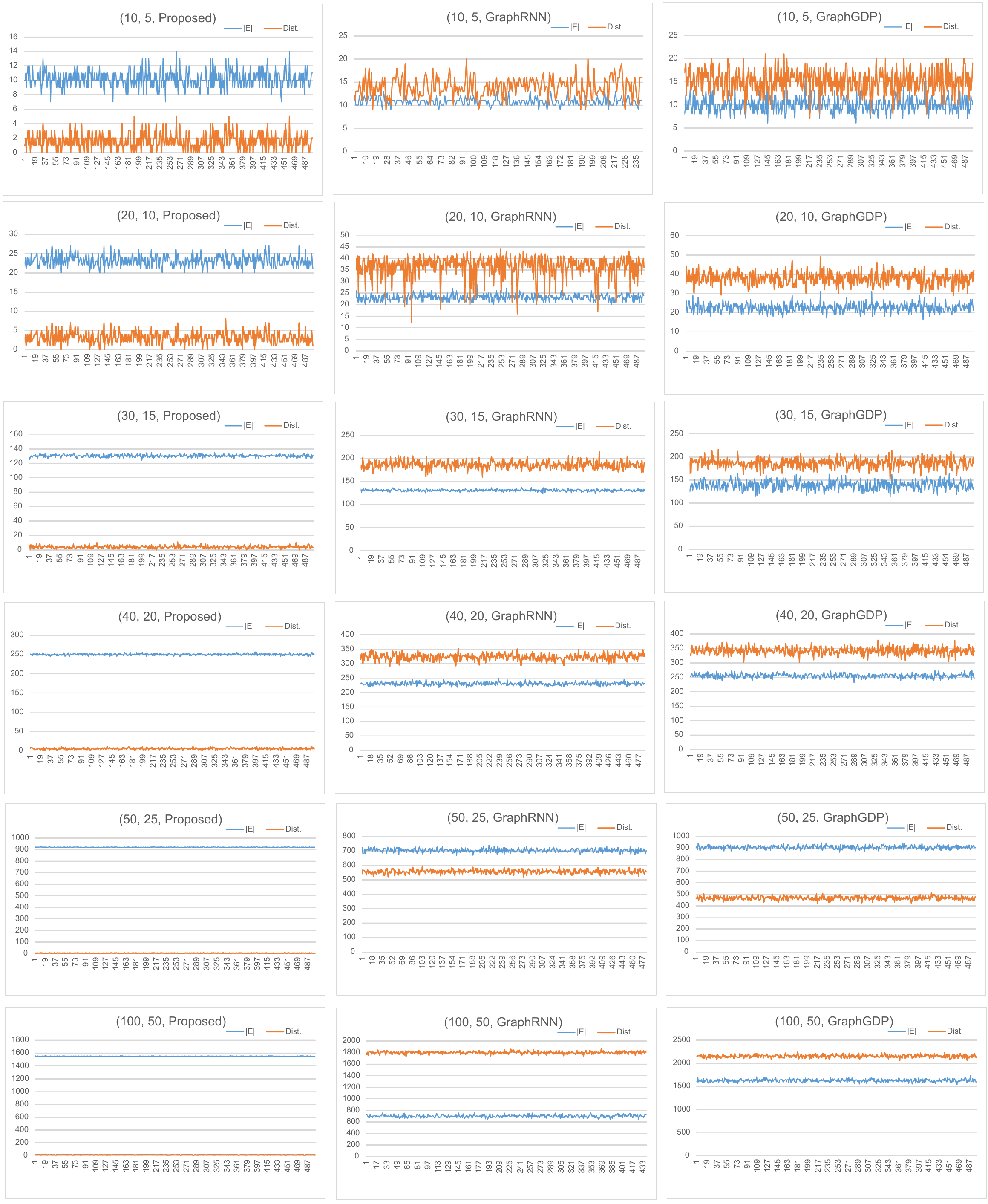}
	\caption{Plots of the number of edges $|E|$ (blue) and the graph edit distances (Dist. in orange) of the graphs with $n$ vertices generated by the proposed network,  GraphRNN~\cite{You2018GraphRNN} and GraphGDP~\cite{Huang2022GraphGDP} 
	for different pairs of $n$ and $d$.}\label{fig:valid_graph}
\end{figure}

\begin{table}[ht]
\centering
\setlength{\tabcolsep}{4pt}

\begin{tabular}{|c|c|c|c|c|c|c|c|c|c|c|c|}
\hline

\multirow{2}{*}{$n$} &
\multirow{2}{*}{$|E|$} &
\multirow{2}{*}{$d$}
& \multicolumn{3}{c|}{Proposed GE$_d$}
& \multicolumn{3}{c|}{GraphRNN~\cite{You2018GraphRNN}}
& \multicolumn{3}{c|}{GraphGDP~\cite{Huang2022GraphGDP}} \\
\cline{4-12}
& & 
& $N_{|E|}$
& $N_n$
& $N_d$
& $N_{|E|}$
& $N_n$
& $N_d$
& $N_{|E|}$
& $N_n$
& $N_d$ \\
\hline

10  & 9     & 5  & 500 & 500 &  500 & 441 & 240 & 0 & 498 & 500 & 0 \\
20  & 20    & 10 & 500 & 500 & 500 & 500 & 499 & 0 & 498 & 500 & 0 \\
30  & 130   & 15 & 500 & 500 & 500 & 500 & 499 & 0 & 369 & 500 & 0 \\
40  & 250   & 20 & 500 & 500 & 500 & 270 & 487 & 0 & 491 & 500 & 0 \\
50  & 918  & 25 & 500 & 500 & 500 & 0 & 483 & 0 & 394 & 500 & 0 \\
100 & 1536  & 50 & 500 & 500 & 500 & 0 & 438 & 0 & 60 & 500 & 0 \\
\hline
\end{tabular}

\caption{Comparison of the proposed network GE$_d$ with GraphRNN~\cite{You2018GraphRNN} and GraphGDP~\cite{Huang2022GraphGDP} models to generate valid graphs with the desired edit distance.}
\label{tab:comparison}
\end{table}
\section{Conclusion} \label{sec:concl}
In this study, we investigated the existence of ReLU-based generative networks capable of generating graphs similar to a given graph under the graph edit distance metric. In our framework, a vertex-labeled graph is represented by its label and adjacency matrices, which serve as both the input and output of the proposed architecture. We showed that graphs obtained through substitution, deletion, or insertion operations within a bounded edit distance can be generated by constant depth ReLU networks whose size scales as $\mathcal{O}(n^2 d)$. Furthermore, we established that, when these operations are combined, a constant depth ReLU network of size $\mathcal{O}(n^2 d)$ can generate any graph within edit distance $d$ from the input graph, providing a deterministic generative model for generating the desired graphs.

{The scalability experiments indicate that the running time of the proposed ${\rm GE}_d$ network increases with both the number of vertices $n$ and the edit distance $d$, with $d$ having a stronger impact on the computational cost. Nevertheless, the method successfully handles graphs with up to 1400 vertices for moderate edit distances up to 140 in a reasonable time, demonstrating good scalability for large graph instances. The comparative experiments show that the proposed GE$_d$ network consistently generates valid graphs, with all generated graphs satisfying the vertex, edge, and edit distance constraints for every tested configuration. In contrast, GraphRNN by You et al.~\cite{You2018GraphRNN} and GraphGDP by Huang et al.~\cite{Huang2022GraphGDP} fail to produce any graph satisfying the required edit distance in all cases, and their structural validity deteriorates for larger graphs (e.g., only 438/500 graphs preserve the vertex count and 60/500 satisfy the edge constraint when $n=100$). These results demonstrate the clear advantage of the proposed method for constrained graph generation.}

{Future research may explore several directions to further enhance the proposed approach. First, reducing the size or depth of the neural network architecture could improve computational efficiency and scalability while maintaining expressive power. Second, developing techniques that enable the uniform generation of graphs, ensuring that each feasible graph is produced with approximately equal probability, would be an important advancement. Finally, extending the framework to real-world graph generation tasks, such as social, biological, and communication networks, would help demonstrate its broader practical applicability.}

An implementation of the proposed networks is available at~\url{https://github.com/MGANN-KU/GraphGen\_ReLUNetworks}.

\section*{Author contributions}
Conceptualization, M.G. and T.A.; methodology, M.G. and T.A.; software, M.G.  validation, M.G.; formal analysis, M.G. and T.A.; investigation, M.G. data curation, M.G.; writing—original draft preparation, M.G.; writing—review and editing, M.G. and T.A..; supervision, T.A.; project administration, T.A. All authors have read and agreed to the published version of the manuscript.

\section*{Acknowledgments }
The work of Tatsuya Akutsu was supported in part by Grants 22H00532 and 22K19830 from Japan Society for the Promotion of Science
(JSPS), Japan. 
The authors would like to thank Dr. Naveed Ahmed Azam, Quaid-i-Azam University Pakistan, for the useful technical discussions.

\section*{Conflict of interest}
All authors have no conflicts of interest in this paper.

\bibliographystyle{unsrt}  
\bibliography{relu_ref}

\end{document}